\documentclass[11pt]{article}
\usepackage{graphicx}
\newcount\Comments
\usepackage[colorlinks=true, allcolors=blue]{hyperref}
\usepackage{bbm}
\usepackage{amsmath}
\usepackage{amssymb}
\usepackage{algorithm}
\usepackage[framemethod=TikZ]{mdframed}
\usepackage{mathtools}
\usepackage{amsthm}
\usepackage{natbib}
\usepackage{letltxmacro}
\usepackage{authblk}
\usepackage{mathrsfs} 

\usepackage{booktabs}
\usepackage{makecell}

\PassOptionsToPackage{algo2e,ruled, linesnumbered}{algorithm2e}
\RequirePackage{algorithm2e}

\LetLtxMacro\amsproof\proof
\LetLtxMacro\amsendproof\endproof

\usepackage{thmtools}
\usepackage{xcolor}
\AtBeginDocument{%
  \LetLtxMacro\proof\amsproof
  \LetLtxMacro\endproof\amsendproof
}

\newtheorem{theorem}{Theorem}

\newtheorem{definition}[theorem]{Definition}
\newtheorem{lemma}[theorem]{Lemma}
\newtheorem{claim}[theorem]{Claim}
\newtheorem{corollary}[theorem]{Corollary}

\newtheorem{example}{Example}
\newtheorem*{theorem*}{Theorem}
\newtheorem*{remark*}{Remark}

\usepackage{xpatch}
\xpatchcmd{\proof}{\itshape}{\normalfont\proofnamefont}{}{}

\usepackage[most]{tcolorbox}

\usepackage[letterpaper,top=2cm,bottom=2cm,left=3cm,right=3cm,marginparwidth=1.75cm]{geometry}

\usepackage[colorlinks=true, allcolors=blue]{hyperref}

\usepackage{bbm}

\usepackage{algorithm}
\usepackage{algpseudocode}

\let\S\undefined
\let\H\undefined

\newcommand{\E}{\mathbbm{E}}
\newcommand{\N}{\mathbbm{N}}
\newcommand{\A}{\mathcal{A}}
\newcommand{\X}{\mathcal{X}}
\newcommand{\Y}{\mathcal{Y}}
\newcommand{\Z}{\mathbbm{Z}}

\newcommand{\S}{\mathcal{S}}
\newcommand{\Q}{\mathcal{Q}}
\newcommand{\M}{\mathcal{M}}
\newcommand{\H}{\mathcal{H}}
\newcommand{\T}{\mathcal{T}}
\newcommand{\1}{\mathbbm{1}}
\newcommand{\PP}{\mathcal{P}}

\newcommand{\vis}{\rm vis}

\newcommand{\prob}{\mathsf{P}}
\newcommand{\cc}{{\rm c}}
\newcommand{\y}{\mathbf{y}}
\newcommand{\x}{\mathbf{x}}
\newcommand{\uu}{\mathbf{u}}

\newcommand{\h}{\mathbf{h}}
\newcommand{\dd}{\mathop{}\!\mathrm{d}}

\DeclareMathOperator*{\argmin}{arg\,min}
\providecommand{\abs}[1]{\lvert#1\rvert}

\let\qedsymbol\undefined
\newcommand{\qedsymbol}{$\blacksquare$}

\newenvironment{proofsketchof}[1]{%
\vspace{0.1cm}
  \par\noindent\textit{Proof sketch of #1.}\quad
}{%
  \hfill\qedsymbol\par
  \vspace{0.1cm}
}

\newenvironment{proofsketch}{%
\vspace{0.1cm}
  \par\noindent\textit{Proof sketch.}\quad
}{%
  \hfill\qedsymbol\par
  \vspace{0.1cm}
}

\title{Partial Feedback Online Learning}

\author{Shihao Shao$^{1}$, Cong Fang$^{2}$, Zhouchen Lin$^{2}$, and Dacheng Tao$^{1}$}
\affil{$^{1}$College of Computing and Data Science, Nanyang Technological University}
\affil{$^{2}$School of Intelligence Science and Technology, Peking University}
\affil{\texttt{\{shihao.shao, dacheng.tao\}@ntu.edu.sg}}
\affil{\texttt{\{fangcong, ZLIN\}@pku.edu.cn}}

\date{}

\begin{document}

\maketitle

\begin{abstract}
We study a new learning protocol, termed \emph{partial-feedback online learning}, where each instance admits a set of acceptable labels, but the learner observes only one acceptable label per round. We highlight that, while classical version space is widely used for online learnability, it does not directly extend to this setting. We address this obstacle by introducing a \emph{collection version space}, which maintains sets of hypotheses rather than individual hypotheses. Using this tool, we obtain a tight characterization of learnability in the set-realizable regime. In particular, we define the Partial-Feedback Littlestone dimension (PFLdim) and the Partial-Feedback Measure Shattering dimension (PMSdim), and show that they tightly characterize the minimax regret for deterministic and randomized learners, respectively. We further identify a nested inclusion condition under which deterministic and randomized learnability coincide, resolving an open question of \cite{onlinelearningsetvaluefeedback}. Finally, given a hypothesis space $\H$, we show that beyond set realizability, the minimax regret can be linear even when $|\H|=2$, highlighting a barrier beyond set realizability.
\end{abstract}

\section{Introduction}

In many deployed pipelines, an instance may admit multiple acceptable labels (e.g., an image containing many animals could be correctly labeled with any one of the animals present).
Yet the system typically logs only a single witness label per instance; it is unrealistic to expect access to all correct labels.
Thus, after predicting, the learner sees one acceptable label but not whether its prediction was acceptable, making online regret minimization with such partial feedback the natural objective.
This partial-feedback phenomenon is superficially reminiscent of partial-label learning (PLL)~\citep{cour2011partial,pll2,tian2023partial}, where each instance is associated with a set of candidate labels. However, PLL typically posits that exactly one candidate is the unique ground-truth label and aims to disambiguate it. In contrast, in our setting multiple labels may be simultaneously acceptable, reflecting intrinsic semantic ambiguity rather than spurious false-positive candidates. We call this setting \emph{partial-feedback online learning}, and provide a thorough characterization of this protocol.

In partial-feedback online learning, each instance admits a (hidden) set of acceptable labels $S_t$, yet the learner only observes a single witness label $y^{\rm vis}_t \in S_t$ after predicting. Crucially, the learner never observes whether its prediction was correct during the game; the sets $S_1,\dots,S_T$ remain hidden during the game and are disclosed only after round $T$ for evaluation. This feedback is different from bandit feedback~\citep{raman2024multiclass} (bandit feedback provides a 1-bit membership query $\1\{ \hat y_t\in S_t \}$; we do not even get that bit) and immediate set-feedback protocols~\citep{onlinelearningsetvaluefeedback} (which reveal $S_t$ immediately).

\noindent\textbf{Protocol with Deterministic Learners.}
Let $\X$ be the instance space and $\Y$ the label space. The game lasts $T$ rounds. 
In each round $t\in[T]$: (1) The adversary chooses an instance $x_t\in\X$ and a \emph{hidden} acceptable-label set 
    $S_t\in \S(\Y)\subseteq \sigma(\Y)$, where $\S(\Y)$ is a prescribed family of allowable acceptable sets, and $\sigma(\Y)$ is a sigma algebra on $\Y$.
(2) The learner predicts a label $\hat y_t\in\Y$.
(3) The adversary reveals only a \emph{single witness} $y^{\rm vis}_t\in S_t$. Notice that
for intuition we describe the protocol using deterministic predictions $\hat y_t$. Throughout the paper we allow the learner to randomize: in round $t$, it outputs a distribution $\hat\pi_t \in \prob(\Y,\sigma(\Y))$ and predicts $\hat y_t \sim \hat \pi_t$. Deterministic learners correspond to delta distributions. We formalize the game, regret notion, and the adversary model in Section~\ref{sec:fd}.
Crucially, during play the learner observes neither the set $S_t$ nor the correctness indicator 
$\1\{\hat y_t\in S_t\}$. After $T$ rounds, the adversary reveals $S_1,\ldots,S_T$ for evaluation, and the learner incurs
$\ell_T := \sum_{t=1}^T \1\{\hat y_t\notin S_t\}$.

\noindent\textbf{Realizability regimes.}
Our main focus is the set-realizable regime, where there exists a fixed subset $F^{\star}\in \H$ such that $S_t = F^{\star}(x_t)$ for all $t$, where $F(x) := \{ f(x) : f\in F\}\subseteq \Y$. Intuitively, the ambiguity comes from a fixed latent set of consistent predictors $F^{\star}$. For example, in single-label object tagging, let $\X$ be images and $\Y$ a finite set of object categories, and suppose
$F^{\star}=\{f_{\mathrm{center}},f_{\mathrm{largest}},f_{\mathrm{near}}\}\subseteq \Y^{\X}$ consists of fixed deterministic labeling rules:
$f_{\mathrm{center}}(x)$ returns the category of the object instance closest to the image center,
$f_{\mathrm{largest}}(x)$ returns the category of the largest object by pixel area, and
$f_{\mathrm{near}}(x)$ returns the category of the object closest to the observer (smallest depth).
Then $S_t=F^{\star}(x_t)$ is the set of labels that these rules would assign to $x_t$, while the log reveals only a single witness label $y_t^{\mathrm{vis}}\in S_t$ (one annotator's choice).
Each round reveals only one witness label from the set of labels that $F^{\star}$ would produce on $x_t$.
We also show that beyond set realizability (e.g., existence realizability or the agnostic setting), the minimax regret can be linear even when the hypothesis space contains only two elements, highlighting an inherent barrier under partial feedback and its differences compared to previous online learning protocols.

\noindent\textbf{Hypotheses and regret.}
Let \(\H \subseteq \Y^\X\) be a hypothesis class of predictors \(f:\X \to \Y\). We measure performance relative to the best predictor using the notion of \textit{regret}, defined as the difference between the learner's total loss and that of the best function $f^{\star}$ in \(\H\), under the fiercest adversary against a learner \(\A\). We say that the problem is online learnable if the minimax regret---namely, the regret of the best learner $\A^{\star}$ against the fiercest adversary---grows sublinearly in \(T\). In the set-realizable regime, the loss of the best function in $\H$ is \(0\), achieved by some \(f^{\star} \in F^{\star}\) for a fixed subset \(F^{\star} \subseteq \H\).

\noindent\textbf{Beyond version space.} 
The classical version space cannot directly extend to our protocol. Under full-information set feedback, a standard approach is to prune $\H$ using the observed pairs $(x_i,S_i)$ by eliminating any $f\in\H$ with $f(x_i)\notin S_i$ for some $i\le t$. The surviving hypotheses form the classical \emph{version space}~\citep{mitchell1977version,mitchell1982generalization}.
In our protocol, however, the learner observes only $y^{\rm vis}_t\in S_t$, which typically provides only \emph{existence} information and does not falsify any single hypothesis.
For example, if $\H=\{f_0\equiv 0,f_1\equiv 1\}$ and $y^{\rm vis}_t=0$, then $f_1$ cannot be ruled out because the underlying acceptable set could be $S_t=\{0,1\}$, which is induced by $F=\{f_0,f_1\}$. The key obstacle is that observing $y^{\rm vis}_t\in S_t$ only certifies that some acceptable function exists, not falsify any particular one. 

\noindent\textbf{Collection version space.}
To recover monotone progress, we lift the classical version space from a set of hypotheses to a collection of hypothesis subsets. We maintain a collection-valued version space \(\tilde V_t \subseteq \tilde{\H}\), consisting of all hypothesis subsets \(F\) that remain consistent with the observed witnesses, with \(\tilde V_0\) initialized as \(\tilde \H\):
\[
\tilde V_t := \Bigl\{ F \in \tilde V_{t-1} : y^{\rm vis}_t \in F(x_t) \Bigr\},
\]
where $\tilde{\H}$ denotes the power set of \(\H\).
Unlike standard pruning, which fails under witness-only feedback, $\tilde V_t$ shrinks whenever a subset cannot explain an observed witness. This restores monotone elimination (now at the level of subsets), enabling tree-based shattering arguments analogous to \cite{littlestone}’s analysis.

\noindent\textbf{Contributions.}
We ask: when is sublinear regret possible under this feedback, and what combinatorial parameters govern the optimal rate? Our main contributions are listed below.

\begin{itemize} \item \emph{(New dimensions + tight minimax bounds)} We propose Partial-Feedback Littlestone dimension (PFLdim) and the Partial-Feedback Measure Shattering dimension (PMSdim) defined on collection version space, and show that they tightly characterize the minimax regret for deterministic and randomized learners, respectively.  

\item \emph{(When determinism = randomization)} We show that deterministic and randomized learnability coincide whenever $\S(\Y)$ satisfies a nested-inclusion property, a structural monotonicity condition on $\S(\Y)$, independent of the finiteness of the Helly number. This also answers \cite{onlinelearningsetvaluefeedback}'s open question on whether finite Helly number is necessary for the learnability equivalence.

\item \emph{(Sharp impossibility beyond set realizability)} We show that beyond set realizability, the minimax regret can be linear even when $|\H|=2$, motivating the development of noise-sensitive complexity measures for partial-feedback learning.

\end{itemize}

\noindent\textbf{Technical Challenge.} Classical online dimensions (e.g., Littlestone dimension~\citep{littlestone}) are defined via labeled trees where the adversary’s feedback is annotated edge-wise and the learner’s prediction corresponds to choosing an outgoing edge. In our setting, the learner observes only witness labels $y^{\vis}$, while the true acceptable sets are revealed only after $T$ rounds. Consequently, the appropriate tree constructions must annotate edges with witness labels but allow the underlying acceptable-set assignment to be specified path-wise, which leads to a tailored tree-based definitions of PFLdim and PMSdim over the collection version space.

\section{Related Works}

Our work connects to several lines of research in online learning under adversarial sequences and limited feedback. In full-information online classification, combinatorial/sequential dimensions provide sharp characterizations of learnability and optimal rates: the Littlestone dimension was introduced by \citet{littlestone} in the realizable mistake-bound setting, and the agnostic minimax-regret viewpoint was developed in \citet{bendavid2009agnostic}, with broader links to sequential complexity frameworks such as sequential Rademacher complexity and related minimax analyses \citep{rakhlin2011online,rakhlin2015sequential}. Multiclass online learning exhibits additional subtleties; classical structural results relate learnability and ERM-style principles \citep{daniely2011littlestone}, and the gap between full-information and bandit feedback has been quantified via the price of bandit information \citep{daniely2013price}. Recent work has obtained sharp characterizations even for unbounded label spaces in the full-information model \citep{hanneke2023multiclass} and in the bandit-feedback model via bandit-specific sequential dimensions \citep{raman2024multiclass}. Closest in spirit to our setting is online learning with set-valued feedback, where the environment reveals the entire acceptable set each round and one predicts any label in the set; this model was recently characterized via tailored dimensions and exhibits separations between deterministic and randomized learnability \citep{onlinelearningsetvaluefeedback}. Our partial-feedback protocol is strictly weaker in terms of supervision: only one correct label is revealed while evaluation is still with respect to an unobserved set, placing it in the broader landscape of partial-information online learning such as adversarial bandits \citep{auer2003nonstochastic} and partial monitoring \citep{bartok2014partialmonitoring}, but with an observation structure that is not a direct loss-revealing signal. This also bears some resemblance to online learning with hidden or only partially revealed feedback structure, where the learner cannot fully exploit the information pattern available to the environment~\citep{fdbkgraph,sideobservation}. Another related thread studies whether minimax-optimal performance can be achieved by proper or otherwise restricted predictors; recent work provides combinatorial characterizations of minimax in $0/1$ games with simple predictors \citep{hanneke2021proper}, and related progress studies refined rates in nonparametric online learning \citep{daskalakis2022fast} as well as randomized Littlestone-type notions for optimal prediction \citep{filmus2023randomized}. Finally, in the batch setting, ambiguous/partial-label supervision has been studied theoretically, e.g., learning from candidate label sets \citep{cour2011partial} and single-positive multi-label learning where each instance is annotated with only one relevant label or other partial information \citep{xu2022onepositive,partialmultilabellearning}; our problem is distinct in being adversarial and sequential with weaker feedback---the learner observes only one revealed correct label while performance is evaluated against an unobserved valid-label set---which motivates new shattering notions and a minimax-regret characterization tailored to this protocol.

\begin{table}[]
\centering
\resizebox{\linewidth}{!}{\begin{tabular}{@{}rccccl@{}}
\toprule
\multicolumn{6}{c}{Partial-Feedback}                                                                                                                                                                                                                                        \\ \midrule
Action & $\Q$ sel. $x_t$     & $\A$ pred. $\hat{\pi}_t$              & $\Q$ rev. $y^{\vis}_t$                        & \makecell{(if $t=T$) \\ $\Q$ rev. $(S_1,\dots,S_T)$} & \multicolumn{1}{c}{\makecell{(if $t=T$) \\ $\ell_T := \sum^T_{i=1} \underset{\hat{y}_i\sim \hat{\pi}_i}{\E}\1\{ \hat{y}_i\notin S_i \}$}} \\ \midrule
Info   & $\Theta_{t-1}$          & $\Theta_{t-1} \bigcup \{ x_t \}$          & $\Theta_{t-1} \bigcup (x_t,\hat{\pi}_t)$          & $\Theta_T$                                            & \multicolumn{1}{c}{$\Theta_T \bigcup (S_1,\dots,S_T)$}                                      \\ \specialrule{1pt}{3pt}{3pt}
\multicolumn{6}{c}{Multiclass~\citep{littlestone} \& Set-valued~\citep{onlinelearningsetvaluefeedback}}                                                                                                                                                                                                                                \\ \midrule
Action & $\Q$ sel. $x_t$     & $\A$ pred. $\hat{\pi}_t$              & $\Q$ rev. $S_t$                               & $\ell_t:=\ell_{t-1} + \underset{\hat{y}_t\sim \hat{\pi}_t}{\E}\1\{ \hat{y}_t \notin S_t \} $       &                                                                                              \\ \midrule
Info   & $\Theta^{\prime}_{t-1}$ & $\Theta^{\prime}_{t-1} \bigcup \{ x_t \}$ & $\Theta^{\prime}_{t-1} \bigcup (x_t,\hat{\pi}_t)$ & $\Theta^{\prime}_{t}$                               &                                                                                              \\ \bottomrule
\end{tabular}}

\vspace{-0.1cm}

\caption{The learning protocol of partial-feedback compared with multiclass \& set-valued settings, where $\A$ is learner, $\Q$ is adversary, histories $\Theta_{i}:=(x_1,\hat{\pi}_1,y^{\vis}_1),\dots,(x_{i},\hat{\pi}_{i},y^{\vis}_{i})$, and $\Theta^{\prime}_{i}:=(x_1,\hat{\pi}_1,S_1),\dots,(x_{i},\hat{\pi}_{i},S_{i})$. In our model, the set of correct labels $S_t$ is fixed by the adversary in advance but hidden from the learner.}
\label{tab:setting}
\end{table}
\vspace{-0.35cm}

\section{Formal Definition}
\label{sec:fd}

We formalize the partial-feedback online learning game and fix notation.
The game is specified by an instance space $\X$, a label space $\Y$ equipped with a $\sigma$-algebra $\sigma(\Y)$,
a prescribed family of admissible acceptable-label sets $\S(\Y)\subseteq \sigma(\Y)$, a hypothesis class
$\H\subseteq \Y^{\X}$, and a horizon $T\in\mathbb N_{+}$.

\noindent\textbf{Protocol.}
For each round $t\in[T]$ ($[T]:=\{1,2,\dots,T\}$, the adversary chooses an instance $x_t\in\X$ and a hidden acceptable-label set $S_t\in\S(\Y)$.
The learner then outputs a distribution $\hat\pi_t\in\prob(\Y,\sigma(\Y))$.
Finally, the adversary reveals only a single witness label $y_t^{\rm vis}\in S_t$.
During the game, the learner observes neither $S_t$ nor the loss $\E_{\hat y_t \sim \hat \pi_t}\1\{\hat y_t\notin S_t\}$.
After round $T$, the adversary reveals $S_{1:T}$ for evaluation.
Let the partial-feedback history be
$\Theta_t := (x_1,\hat\pi_1,y^{\rm vis}_1),\ldots,(x_t,\hat\pi_t,y^{\rm vis}_t)$ (with $\Theta_0$ empty).
A (possibly randomized) learner is a sequence of maps $\A:=\{\A_t\}_{t=1}^T$ such that
$\hat\pi_t=\A_t(\Theta_{t-1},x_t)\in\prob(\Y,\sigma(\Y))$. The learner becomes deterministic if it only outputs delta measures. 

\noindent\textbf{Loss, regret, and learnability.}
The deterministic round-$t$ loss is $\1\{\hat y_t\notin S_t\}$; for randomized learners we evaluate the expected loss
$\E[\1\{\hat y_t\notin S_t\}]$ where the expectation is over $\hat y_t\sim\hat\pi_t$.

\begin{definition}[Partial-feedback online learnability]
\label{def:learnability_pf}
For a learner $\A$ and a sequence $(x_t,S_t,y_t^{\rm vis})_{t=1}^T$ satisfying $S_t\in\S(\Y)$ and $y_t^{\rm vis}\in S_t$,
define the learner's expected cumulative loss and the comparator loss by $\ell_T(\A):=\sum_{t=1}^T \E[\1\{\hat y_t\notin S_t\}]$ and $\ell_T(f):=\sum_{t=1}^T \1\{f(x_t)\notin S_t\}\ \ (f\in\H)$, respectively.
The worst-case expected regret of $\A$ on $\H$ is
\begin{equation}
\label{eq:defofregret}
R_{\A}(T,\H):=\sup_{\substack{(x_t,S_t,y_t^{\rm vis})_{t=1}^T\\ S_t\in\S(\Y),\ y_t^{\rm vis}\in S_t}}
\bigl(\ell_T(\A)-\inf_{f\in\H}\ell_T(f)\bigr),
\end{equation}
and the minimax regret is $R^\star(T,\H):=\inf_{\A}R_{\A}(T,\H)$, where the infimum is over all (possibly randomized) learners.
We say that $\H$ is \emph{partial-feedback online learnable} (with randomized learners) if $R^\star(T,\H)=o(T)$. A learner is \emph{deterministic} if $\hat\pi_t$ is always a delta distribution $\delta_{\hat y_t}$.
Let $R^\star_{\rm det}(T,\H):=\inf_{\A\in\textsf{Det}}R_{\A}(T,\H)$; we say $\H$ is \emph{deterministically learnable}
if $R^\star_{\rm det}(T,\H)=o(T)$.
\end{definition}
\begin{remark*}
Under set-realizable assumption, minimizing $R_{\A}(T,\H)$ among $\A$ in Equation~\eqref{eq:defofregret} can be rewritten as
\begin{equation}
\label{eq:regretrewritten}
\begin{aligned}
R^{\star}(T,\H) &= \inf_{\A}
\sup_{\substack{
(x_i,S_i,y_i^{\rm vis})_{i=1}^T\\
S_i\in\S(\Y),\ y_i^{\rm vis}\in S_i\\
\exists\,F^\star\in\tilde\H\ \text{s.t.}\ S_i=F^\star(x_i),\ \forall i\in[T]
}}
\sum_{t=1}^T
\underset{{\hat y_t\sim \A_t(\Theta_{t-1},x_t)}}{\E}\!\left[\1\{\hat y_t\notin S_t\}\right]\\
&= \sup_{x_1}\inf_{\hat\pi_1}\sup_{y_1^{\rm vis}} \sup_{x_2}\inf_{\hat\pi_2}\sup_{y_2^{\rm vis}} \cdots \sup_{x_T}\inf_{\hat\pi_T}\sup_{y_T^{\rm vis}} \\ &\qquad \sup_{\substack{ \exists\,F^\star\in\tilde{\H}\ \text{s.t.}\ S_i=F^\star(x_i),\ \forall i\in[T]\\ S_i\in \S(\Y),\ y_i^{\rm vis}\in S_i,\ \forall i\in[T] }} \sum_{t=1}^T \E_{\hat y_t\sim \hat\pi_t}\bigl[\1\{\hat y_t\notin S_t\}\bigr].
\end{aligned}
\end{equation}
The minimax regret can be viewed as a zero-sum game between the learner and the adversary.
For this worst-case analysis, it is convenient to use an equivalent protocol in which, at each round,
the adversary first chooses $x_t$, then the learner outputs a distribution $\hat{\pi}_t$, and then the adversary chooses the revealed witness $y_t^{\rm vis}$.
The hidden acceptable sets $S_{1:T}$ are specified only at the end, subject to
$y_t^{\rm vis}\in S_t$ for all $t$. Since the learner never observes $S_t$ during play,
this postponement does not affect the learner's observable history or the resulting minimax regret. An illustration of this protocol as well as comparison with others is shown in Table~\ref{tab:setting}.
\end{remark*}

\noindent\textbf{Realizability regimes.}
In the \emph{agnostic} regime no further assumptions are imposed on $S_{1:T}$.
In the \emph{existence-realizable} regime there exists a fixed $f^\star\in\H$ such that $f^\star(x_t)\in S_t$ for all $t$.
In the \emph{set-realizable} regime there exists a fixed subset $F^\star\subseteq\H$ such that
$S_t=F^{\star}(x_t)$ for all $t$.
We characterize the learnability under set-realizable setting before Section~\ref{sec:arrs}.

\section{Deterministic Learners}
\label{sec:det}

\subsection{A First Attack on Finite-class Learnability}

As a motivating example, we show that collection version space helps in yielding an (admittedly loose) finite-class learnability guarantee.

\begin{theorem}[Finite-Class Learnability (Coarse Bound)]
\label{thm:finitebound}
Under partial feedback and set-realizability, any finite $\H$ with $|\H|=n$ is online learnable by a deterministic algorithm with
minimax regret at most $\sum_{i=1}^{\lfloor n/2\rfloor}\binom{n}{i}$.
\end{theorem}
\begin{proofsketch}
We maintain and update the collection version space, based on $\tilde V_t := \{ F\in \tilde V_{t-1} \mid y^{\rm vis}_t \in F(x_t) \}$. At round $t$, the learner uses a for-loop to search for $y\in\Y$ such that $y\in F(x_t)$ for all $F\in \tilde V_{t-1}$, and predicts it as $\hat y_t$. If the prediction is made inside this for-loop, the learner will be correct. If the prediction is not made inside this for-loop, the learner predicts $\hat y_t$ that is output by the majority of functions in $\H$ at $x_t$. We upper bound the wrong predictions outside the for-loop. Let $V^{\rm vis}_t := \{ f\in \H \mid f(x_t) = y^{\rm vis}_t \}$ be the set of functions that output $y^{\rm vis}_t$ at $x_t$. We show that $V^{\rm vis}_t\neq V_{t^{\prime}}^{\rm vis}$ for all $t < t^{\prime}$, or the prediction at $t^{\prime}$ will be made inside the for-loop and will be correct. According to the majority-prediction strategy, $\abs{ V^{\rm vis}_t } \le \lfloor \frac{n}{2} \rfloor$ or $y^{\rm vis}_t = \hat y_t$ and the learner will be correct. Counting all such $V^{\rm vis}_t$ gives us an upper bound of $\sum_{i=1}^{\lfloor n/2\rfloor}\binom{n}{i}$.
\end{proofsketch}
It highlights that partial feedback is non-trivial even for finite classes. Although collection version space makes this problem solvable, and helps in proving the learnability (regret constant in $T$), the bound is exponential in $n$. This motivates our sharper, dimension-based characterizations.

\subsection{Partial-Feedback Littlestone Dimension}
 To characterize the learnability of a specific setting, it is conventional to propose a tailored dimension, and prove that the minimax regret can be tightly bounded by this specific dimension. In the finite hypothesis space ($\H<\infty$) case, less functions in $\H$ means easier learning, which are also evident in the previous finite case analysis (Theorem~\ref{thm:finitebound}). Likewise, we need a quantity to justify the complexity of the hypothesis class whether finite or infinite, and this is where the combinatorial dimensions come into play. Larger dimensions correspond to harder learning problems. The combinatorial dimensions for online learning (including ours) are typically defined via rooted $\Y$-ary trees whose nodes are labeled by instances in $\X$ and edges carry additional annotations. 

We therefore fix the following notations.
An $\X$-valued $\Y$-ary tree of depth $d$ is a mapping $\T:\Y^{< d}\to \X$, where each node at depth $t$ is indexed by a prefix $\y_{1:t-1}\in\Y^{t-1}$ and labeled by the instance $\T(\y_{1:t-1})\in\X$.
Specifically, one may view a \emph{full} $\T$ as a map $\T:\Y^\star\to \X\cup\{\varnothing\}$ such that $\T(\y)\neq\varnothing$ for all $|\y|<d$ and $\T(\y)=\varnothing$ for all $|\y|\ge d$.
We also allow edge annotations: for each $t\in[d]$, an annotation function $f_t:\Y^{t}\to \M$ assigns an element of $\M$ to the edge corresponding to the length-$t$ path $\y_{1:t}$. 

\subsubsection{Review of Related Dimensions}
Before introducing our dimension, we briefly review the tree-based dimensions most relevant to our setting. This will clarify which parts of the classical definitions survive under partial feedback and which must be modified.

\begin{definition}[Set, Multiclass, and Bandit Littlestone Dimension \citep{onlinelearningsetvaluefeedback}] 
\label{def:smbldim}
    Let $\T$ be a depth-$d$ $\X$-valued, $\Y$-ary full tree with annotation functions $f_t:\Y^{t}\longrightarrow \S(\Y)$ for $t\in [d]$, and $\H$ be the hypothesis class we consider. If there exists a sequence, $(f_1,\dots, f_d)$, such that for every root-to-leaf path $\y_{1:d}=(y_1,\dots,y_d) \in \Y^d$, there exists $h_{\y_{1:d}}\in \H$, such that $h_{\y_{1:d}}(\T(\y_{1:t-1}))\in f_t(\y_{1:t})$ and $y_t\notin f_t(\y_{1:t})$, then we claim that $\T$ is shattered by $\H$. The maximal depth $d$ of $\T$ that can be shattered by $\H$ is defined as the Set Littlestone Dimension (SLdim). If $\S(\Y) = \{\{y\}:y\in \Y\}$ holds, then it is called the Multiclass Littlestone Dimension (MLdim), and if $\S(\Y) =\{A_y\subseteq \Y\ \mid \  A_y=\Y\setminus\{y\}, \ y\in \Y \}$ holds, then it is called the Bandit Littlestone Dimension (BLdim).
\end{definition}
\begin{remark*}
The above definition unifies several standard online classification models by choosing different
set systems $\S(\Y)$. In particular, when $\S(\Y)=\{\{y\}:y\in\Y\}$ consists of singletons,
the induced quantity coincides with the original Littlestone dimension (Ldim)~\citep{littlestone}, i.e., Ldim $\equiv$ MLdim. When
$\S(\Y)=\{\Y\setminus\{y\}:y\in\Y\}$, the definition specializes to the bandit-feedback model,
and the associated dimension reduces to the bandit Littlestone dimension~\citep{raman2024multiclass}.
\end{remark*}

We now further explain these dimensions by clarifying how each part of the definition corresponds to the online learning procedure. The standard online learning protocol is summarized in Table~\ref{tab:setting}. Each root-to-leaf path $\y_{1:d}$ corresponds to a $d$-round game. Along this path, the nodes $\T(\y_{1:t-1})$ are the instances shown by the adversary, the edges $y_t$ are the labels output by the learner after seeing $\T(\y_{1:t-1})$, and the annotation $f_t(\y_{1:t})$ on each edge is the correct label set revealed by the adversary. The requirement that there exists $h_{\y_{1:d}}$ going across all $f_t(\y_{1:t})$ is meant to satisfy the existence-realizability assumption. The other requirement, $y_t\notin f_t(\y_{1:t})$, implies that the learner makes a mistake at the $t$-th round. Combining these together, the existence of a full depth-$d$ tree satisfying these constraints implies that, no matter what deterministic strategy the learner adopts, the adversary can make the learner make mistakes in at least $d$ rounds. In other words, under existence-realizability, there exists an adversary such that any deterministic learner will incur regret of at least $d$.

These combinatorial dimensions uniquely determine the deterministic learnability of the corresponding online learning problem.

\begin{theorem}[Deterministic Learnability Determined by SLdim~\citep{onlinelearningsetvaluefeedback}]
\label{thm:SLdim}
    Consider a set-valued online learning problem. Let the subset collection be $\S(\Y)\subseteq \sigma(\Y)$, and hypothesis class be $\H\in \Y^{\X}$. It holds that $R^\star_{\rm det}(T,\H)={\rm SL}(\H)$ for all $T\geq {\rm SL}(\H)$ under existence-realizable setting.
\end{theorem}
\begin{remark*}
By changing $\S(\Y)$, this learnability conclusion can be applied to multiclass and bandit online learning. 
\end{remark*}

In the next section, we explain why existing tree-based dimensions do not suffice to characterize the partial-feedback online learning protocol, and then introduce a tailored dimension.

\subsubsection{Partial-Feedback Littlestone Dimension}

The Littlestone-like dimension treats each edge as the prediction, and the annotation as the revealed label sets.
The partial-feedback protocol, however, is not directly captured by the existing Littlestone-type dimensions.
The key obstacle is that the learner never observes whether its prediction is correct, nor the full valid-label set $S_t$; it only observes a single witness label $y^{\rm vis}_t\in S_t$.
In the dimensions of Definition~\ref{def:smbldim}, the adversary's feedback along a root-to-leaf path is encoded by edge annotations:
for a fixed prefix $\y_{1:t}$, the annotation (ground-truth label set) at depth $t$ is determined solely by that prefix and therefore is consistent across all extensions $\y_{1:t}\sqcup \y_{t+1:d}$. This implies that the ground-truth label sets in previous rounds are uniquely fixed no matter what the learner predicts in future rounds. 
In contrast, under partial feedback, after observing only $\y^{\rm vis}_{1:t}$ there may remain multiple ground-truth label sets (also, multiple ground-truth function subsets $F^\star\in\tilde\H$) consistent with the history, so the ``unseen part'' of $S_t$ is not pinned down by the prefix.
This motivates a new dimension that explicitly captures such branching of the underlying ground truth under partial revelation, which we define next.

\begin{definition}[Partial-Feedback Littlestone Dimension]
\label{def:pfldim}
Let $\T$ be an $\X$-valued, $\Y$-ary full tree of depth $d$, and let $\H$ be a hypothesis class.
We say that $\T$ is \emph{$q$-shattered by $\tilde{\H}$} if there exists a sequence of annotation functions
$(f^{\rm vis}_1,\ldots,f^{\rm vis}_d)$ with $f^{\rm vis}_t:\Y^{t}\to \Y$ such that for every root-to-leaf path
$\y_{1:d}=(y_1,\ldots,y_d)\in\Y^d$, there exists a function subset $F_{\y_{1:d}}\in \tilde{\H}$ satisfying:
\begin{enumerate}
\item for all $t\in[d]$,
\[
f^{\rm vis}_t(\y_{1:t}) \in F_{\y_{1:d}}\big(\T(\y_{1:t-1})\big)
\quad\text{and}\quad
F_{\y_{1:d}}\big(\T(\y_{1:t-1})\big)\in \S(\Y);
\]
\item
\[
\Big|\big\{t\in[d]: y_t\notin F_{\y_{1:d}}\big(\T(\y_{1:t-1})\big)\big\}\Big|\ge q.
\]
\end{enumerate}
We sometimes say $\y_{1:d}$ is $q$-witnessed by $F_{\y_{1:d}}$ for shorthand when $\y_{1:d}$ meets the above two properties.
For fixed $d$, the \emph{$d$-Partial-Feedback Littlestone Dimension} (PFLdim$_d$) of $\tilde{\H}$ is the maximum $q$
for which there exists a depth-$d$ tree $\T$ that is $q$-shattered by $\tilde{\H}$.
\end{definition}
\begin{remark*}
    PFLdim$_d$ monotonously increases in $d$, as shown in Lemma~\ref{lemma:pfldimmono}.
\end{remark*}
In contrast to the previous dimensions, the annotation function $f^{\rm vis}_t$ outputs the label that the adversary reveals to the learner, instead of the full label set. The crux is that, given two different paths $\y^{(1)}_{1:d}:= \y_{1:k}\sqcup \y^{(1)}_{k+1:d}$ and $\y^{(2)}_{1:d}:= \y_{1:k}\sqcup \y^{(2)}_{k+1:d}$, they can have different ground-truth function sets $F_{\y^{(1)}_{1:d}}\neq F_{\y^{(2)}_{1:d}}$ and therefore different ground-truth label sets, while they even share the same prefix path $\y_{1:k}$. This is, however, not possible in the previous dimensions (e.g., Definition~\ref{def:smbldim}) as we have discussed before. 

We show that our proposed dimension can be used to tightly bound the minimax regret. For any collection version space $\tilde V\subseteq \tilde \H$, we write PFL$_d(\tilde V)$ for its $d$-partial-feedback Littlestone dimension.

\begin{theorem}[Deterministic Learnability Determined by PFLdim]
\label{thm:dldbp}
Fix a subset collection $\S(\Y)\subseteq \sigma(\Y)$ and a hypothesis class $\H\subseteq \Y^{\X}$. Under the set-realizable setting, it holds that
\[
R^{\star}_{\rm det}(T,\H)={\rm PFL}_T(\tilde{\H})
\qquad\text{for all } T\in\N_+.
\]
\end{theorem}
When the previous Littlestone-type dimensions (e.g., ${\rm SLdim}$) are infinite, deterministic learnability in the corresponding classical online learning models fails: for every horizon $T$, an adversary can force $\Omega(T)$ mistakes even in the realizable setting.
In contrast, Theorem~\ref{thm:dldbp} shows that partial-feedback learnability is governed by the growth rate of ${\rm PFL}_T(\tilde{\H})$ as a function of $T$, rather than by whether the dimension is finite.
In particular, even if $\sup_{d\ge1}{\rm PFL}_d(\tilde{\H})=\infty$, the problem can still be learnable provided ${\rm PFL}_d(\tilde{\H})=o(d)$.

\begin{corollary}[Deterministic Learnability via the Growth of ${\rm PFL}_d$]
\label{cor:dldbpcor}
Fix $\S(\Y)\subseteq \sigma(\Y)$ and $\H\subseteq \Y^{\X}$. Then there exists a deterministic online learner with sublinear regret if and only if
\[
{\rm PFL}_d(\tilde{\H})=o(d)\qquad \text{as } d\to\infty.
\]
\end{corollary}
We now outline the proof of Theorem~\ref{thm:dldbp}, which is established via matching upper and lower bounds.
For the upper bound, we construct a deterministic strategy and show that its worst-case regret over $T$ rounds is at most ${\rm PFL}_T(\tilde{\H})$.
A natural first attempt is to mimic the Standard Optimal Algorithm of \citet{onlinelearningsetvaluefeedback},
which at each step predicts a label that minimizes the resulting ${\rm SLdim}$.
However, this approach crucially relies on the set-valued feedback model: the learner can certify whether its prediction is correct and can prune the version space accordingly, and the associated potential (captured by ${\rm SLdim}$) decreases by at least one on each mistake, yielding a mistake bound equal to ${\rm SLdim}$.
In our partial-feedback protocol, by contrast, the learner observes only a single revealed label and cannot infer whether its own prediction was correct, so past mistakes are not observable and the potential-drop argument breaks. To handle this additional uncertainty, we introduce an auxiliary combinatorial quantity that tracks \emph{prefix uncertainty}—informally, it accounts for the set of ways the unobserved parts of the valid-label sets could be consistent with the revealed labels along a history prefix.
This auxiliary dimension is used solely as a proof device to make the inductive argument go through; importantly, it does not appear in the final minimax characterization, which depends only on ${\rm PFL}_T(\tilde{\H})$.

\begin{algorithm}[t]
\caption{Deterministic Partial-Feedback Learning Algorithm (DPFLA)}
\label{alg:dpfla}
\SetKwInOut{Initialize}{Initialize}

\Initialize{$\tilde V_0 = \tilde{\H}$; $\x_{1:0}=()$; $\hat{\y}_{1:0}=()$;}

\For{$t=1,\ldots,T$}{
  receive $x_t\in\X$;

  update $\x_{1:t} \gets \x_{1:t-1} \sqcup (x_t)$;

  let $Q_t \gets \bigcup\limits_{F\in \tilde{V}_{t-1}} F(x_t) $;

  \ForEach{$y\in Q_t$}{
      $\tilde V_{t-1}^{(x_t,y)} \gets \{F\in \tilde V_{t-1}:\exists\, f\in F \text{ such that } f(x_t)=y\}$;
  }

  predict
  \begin{equation}
  \label{eq:choiceofdpfla}
    \hat{y}_t \in \argmin_{\overline{y}_t\in\Y}\ \max_{\overline y^{\rm vis}_t\in Q_t}\ 
    {\rm PPFL}_{T-t}\!\left(\x_{1:t},\ \hat{\y}_{1:t-1} \sqcup (\overline{y}_t),\ \tilde V_{t-1}^{(x_t,\overline y^{\rm vis}_t)}\right);
  \end{equation}

  update $\hat{\y}_{1:t} \gets \hat{\y}_{1:t-1} \sqcup (\hat{y}_t)$;

  receive $y^{\rm vis}_t\in \Y$ \tcp*{with $y^{\rm vis}_t\in Q_t$}

  update $\tilde V_t \gets \tilde V^{(x_t,y^{\rm vis}_t)}_{t-1}$;
}
\end{algorithm}

\begin{definition}[Prefix Partial-Feedback Littlestone Dimension (PPFLdim)]
\label{def:ppfldim}
Fix an integer $n\ge 0$, a history of instances $\x_{1:n}=(x_1,\ldots,x_n)\in\X^n$, and the corresponding predicted labels $\y_{1:n}=(y_1,\ldots,y_n)\in\Y^n$. Let $\tilde V\subseteq \tilde H$ be a collection version space, and let $\T$ be an $\X$-valued, $\Y$-ary full tree of depth $d$, representing a $d$-round continuation after the fixed history $(\x_{1:n},\y_{1:n})$.

We say that $\T$ is \emph{$(\x_{1:n},\y_{1:n})$-prefix-$q$-shattered with respect to $\tilde V$} if there exists a sequence of annotation functions
\[
(f^{\rm vis}_{n+1},\ldots,f^{\rm vis}_{n+d}),
\qquad
f^{\rm vis}_t:\Y^t\to\Y,
\]
such that for every continuation path $\y_{n+1:n+d}\in\Y^d$, letting
\[
\y_{1:n+d}:=(\y_{1:n}\sqcup \y_{n+1:n+d}),
\]
there exists $F_{\y_{1:n+d}}\in\tilde V$ satisfying:

\begin{enumerate}
\item[(i)] \textbf{Consistency with revealed labels.} For all $t\in\{n+1,\ldots,n+d\}$,
\[
f^{\rm vis}_t(\y_{1:t}) \in F_{\y_{1:n+d}}\!\big(\T(\y_{n+1:t-1})\big)
\quad\text{and}\quad
F_{\y_{1:n+d}}\!\big(\T(\y_{n+1:t-1})\big)\in \S(\Y).
\]

\item[(ii)] \textbf{At least $q$ total non-membership events.} The total number of indices on which the predicted label is not contained in the corresponding valid-label set, counting both the fixed history and the continuation, is at least $q$, i.e.,
\[
\Big|\{k\in[n]: y_k\notin F_{\y_{1:n+d}}(x_k)\}\Big|
\;+\;
\Big|\{t\in\{n+1,\ldots,n+d\}: y_t\notin F_{\y_{1:n+d}}\!\big(\T(\y_{n+1:t-1})\big)\}\Big|
\;\ge\; q.
\]
\end{enumerate}

For fixed $(\x_{1:n},\y_{1:n})$ and $d$, the \emph{$(\x_{1:n},\y_{1:n})$-prefix $d$-PPFLdim of $\tilde V$} is the maximum $q$ for which there exists a depth-$d$ tree $\T$ that is $(\x_{1:n},\y_{1:n})$-prefix-$q$-shattered with respect to $\tilde V$.
\end{definition}
\begin{remark*}
Note that the historical non-membership events are evaluated with respect to the path-dependent subset $F_{\y_{1:n+d}}$. Thus, even after the visible history is fixed, the number of mistakes attributed to the learner on the historical rounds may still vary across different future continuations.
\end{remark*}

For convenience, we denote this quantity by ${\rm PPFL}_d(\x_{1:n},\y_{1:n},\tilde{V})$.
It is immediate from the definitions that, for any $\H$, the prefix dimension ${\rm PPFL}_d(\x_{1:n},\y_{1:n},\tilde{\H})$
reduces to ${\rm PFL}_d(\tilde{\H})$ when the prefix is empty, i.e., $n=0$.
For non-empty prefixes, ${\rm PPFLdim}$ serves as a prefix-aware shattering measure that tracks the remaining ambiguity
consistent with the revealed labels along the observed history, and is the key auxiliary quantity in our inductive analysis.
We are now ready to present the proof sketch of Theorem~\ref{thm:dldbp}.

\begin{proofsketchof}{Theorem~\ref{thm:dldbp}}
We give matching upper and lower bounds.

\paragraph{Upper bound.}
Let $\A$ be the learner adopting Algorithm~\ref{alg:dpfla} (DPFLA). By the definition of PPFLdim, we have that
\begin{equation*}
    R^{\star}_{\rm det}(T,\H) \le R_{\A} (T,\H) \le \max_{\x_{1:T},\y^{\vis}_{1:T}} {\rm PPFL}_0(\x_{1:T},\hat\y_{1:T},\tilde{V}_T),
\end{equation*}
where $\hat\y_j$ is output by $\A$ based on $\x_{1:j}$ and $\y^{\rm vis}_{1:j-1}$; $\tilde V_T$ is based on $\x_{1:T}$ and $\y^{\rm vis}_{1:T}$ according to DPFLA.

To prove $R^{\star}_{\rm det}(T,\H) \le {\rm PFL}_T(\tilde\H)$, it suffices to show by induction on $t$ that for any sequence $\x_{1:T}$ and $\y^{\rm vis}_{1:T}$, the PPFL potential never exceeds the initial PFL value:
\[
{\rm PPFL}_{T-t}(\x_{1:t},\hat\y_{1:t},\tilde V_t)\ \overset{!}{\le}\ {\rm PFL}_T(\tilde{\H})\qquad \forall\,t\in\{0,\ldots,T\}.
\]
The base case $t=0$ holds since PPFLdim reduces to PFLdim on an empty prefix.
For the inductive step, suppose the inequality fails at time $t$. By the min--max prediction rule of DPFLA, this would imply
that for every candidate prediction $\hat y$ there exists a feasible revealed label $y^{\rm vis}_t\in Q_t$ such that the updated state
$\tilde V^{(x_t,y)}_{t-1}$ yields a PPFL value larger by at least $1$ than
${\rm PPFL}_{T-(t-1)}(\x_{1:t-1},\hat\y_{1:t-1},\tilde V_{t-1})$.
Using the witness continuation tree from the PPFL definition for this worst-case choice, one can prepend a new root labeled by
$x_t$ and define the first-edge visible annotation to be the above revealed label $y_t^{\rm vis}$ for each edge labeled with $\hat{y}\in \Y$, thereby constructing a depth
$T-(t-1)$ tree witnessing  $(\x_{1:t-1},\hat\y_{1:t-1})$-prefix $q$-shattering with
$q\ge {\rm PPFL}_{T-(t-1)}(\x_{1:t-1},\hat\y_{1:t-1},\tilde V_{t-1})+1$.
By definition of PPFL, this would force
${\rm PPFL}_{T-(t-1)}(\x_{1:t-1},\hat\y_{1:t-1},\tilde V_{t-1})\ge q$, a contradiction.
Thus the potential is always bounded by ${\rm PFL}_T(\tilde{\H})$, and so is the minimax regret. 

\paragraph{Lower bound.}
Let $q={\rm PFL}_T(\tilde{\H})$, and take a depth-$T$ tree $\T$ and visible-label annotations witnessing that $\T$ is $q$-shattered.
An adversary plays the instance $x_t=\T(\hat\y_{1:t-1})$ (following the learner's prediction path) and reveals
$y_t^{\rm vis}=f_t^{\rm vis}(\hat\y_{1:t})$. By shattering, there exists a subset $F^\star\in \tilde{\H}$ consistent with all
revealed labels (hence the play is set-realizable), yet along the induced path the learner's predictions fall outside
$F^\star(x_t)$ on at least $q$ rounds. Therefore any deterministic learner incurs at least $q$ mistakes.

Combining the two bounds yields $R^{\star}_{\rm det}(T,\H)={\rm PFL}_T(\tilde{\H})$.
\end{proofsketchof}

\subsection{Finite Binary Classes: a Simple Application of PFLdim}

We show that PFLdim can be used to improve the bound in Theorem~\ref{thm:finitebound} under binary classification.
\begin{theorem}[Improved Bound for the Finite Case]
\label{thm:finitecaseimprovedbound}
    Under partial feedback and set-realizable settings, let $\H \in \Y^{\X}$ denote the hypothesis class with $\abs{\H}=n<\infty$ and $\abs{\Y} = 2$. Then, $\H$ is online learnable under deterministic algorithms with minimax regret no more than $n$.
\end{theorem}

\begin{proofsketch}
Since $\S(\Y)\subseteq \PP(\Y)$, and $\S(\Y)$ determines the scope of valid function sets, the minimax regret is obtained when $\S(\Y)$ is maximal, i.e., $\S(\Y)=\PP(\Y)$. Fix $T$ and let $q={\rm PFL}_T(\tilde\H)$. By Lemma~\ref{lemma:noviseqpred}, there is a depth-$\le T$ binary tree $\T$ that is
$q$-shattered by $\tilde{\H}$ with annotations $(f^{\rm vis}_t)$ and the property that along any root-to-leaf path
$y_t\neq f^{\rm vis}_t(\y_{1:t})$.

Using the technique in the proof of Lemma~\ref{lemma:pfldimmono}, we complete $\T$ to a full depth-$T$ binary tree that is still $q$-shattered by $\tilde \H$. For any fixed subset $F\in \tilde \H$, consider depth-$T$ paths that are
$q$-witnessed by $F$ (i.e., $f^{\rm vis}_t(\y_{1:t})\in F(\T(\y_{1:t-1}))$ for all $t\in [T]$, and
$y_t\notin F(\T(\y_{1:t-1}))$ for at least $q$ indices). On each of these $q$ indices, $|\Y|=2$ forces the next edge
(choice of $y_t$) to be unique to keep that $y_t\neq f^{\rm vis}_t(\y_{1:t})$, so $F$ can branch freely on at most $T-q$ rounds and thus
$q$-witness at most $2^{T-q}$ paths. Since every one of the $2^T$ paths is $q$-witnessed by some $F$ and $|\tilde \H|=2^n$, we have
$2^T\le 2^n\cdot 2^{T-q}$, hence $q\le n$. By Theorem~\ref{thm:dldbp}, the minimax regret is at most $n$.
\end{proofsketch}

\subsection{Relation with Other Combinatorial Dimensions}
\label{sec:rolp}
We position ${\rm PFLdim}$ relative to classical dimensions under the same $\X,\Y,\H$, with the subset collection $\S(\Y)$ specified below.
Full proofs, discussions, and separating examples are deferred to Appendix~\ref{sec:app_rolp}.

\begin{theorem}[Partial Feedback vs.\ Multiclass]
\label{thm:pfm_short}
Assume $\S_{\rm mc}(\Y)\subseteq \S_{\rm pf}(\Y)$. Then ${\rm MLdim}(\H)\le {\rm PFLdim}(\tilde{\H})$, and the inequality can be strict with ${\rm MLdim}(\H)<\infty$ and ${\rm PFLdim}(\tilde{\H})=\infty$.
\end{theorem}
\begin{proofsketch}
The strictness is shown in Example~\ref{example:relationtomulticlass_app}, let $\X:=[-\pi,\pi]$, $\Y:=[-1,1]$, and $\H:=\{\,x\mapsto \sin(x+c): c\in[0,2\pi)\,\}$. 
With multiclass feedback, a single labeled example $(x,y)$ restricts $c$
to at most two values (since $\sin(x+c)=y$ has at most two solutions), hence the
resulting version space has size at most $2$ after one round. From that point on,
full-information online learning has a constant mistake bound: if the adversary
ever chooses an $x$ on which the two remaining candidates disagree, the revealed
label identifies the correct hypothesis; if it never does, then the learner is
correct on all subsequent rounds anyway. In contrast, under partial feedback with a richer $\S(\Y)$, even though each
revealed witness label still leaves at most two hypotheses consistent with the
observation, the adversary is not forced to falsify either candidate. Instead,
it can maintain ambiguity across rounds so that any deterministic learner makes
$\Omega(T)$ mistakes (indeed at least $T/2$), showing that $\H$ is not
deterministically set-realizable partial-feedback learnable.
\end{proofsketch}

\begin{theorem}[Partial Feedback vs.\ Set-valued]
\label{thm:pfsv_short}
Assume $\S_{\rm pf}(\Y)\subseteq \S_{\rm sv}(\Y)$ and that $\S_{\rm sv}(\Y)$ is closed under arbitrary unions. If
${\rm PFL}_d(\tilde{\H})/d\to 1$ as $d\to\infty$, then ${\rm SL}(\H) =\infty$; without union-closure, this implication may fail even when the above asymptotic condition holds.
\end{theorem}
\noindent Proof and a counterexample for the non-union-closed case appear in Appendix (see Example~\ref{ex:pf_not_sv}), showing that union-closed property is not redundant for guaranteeing ${\rm SL}(\H) =\infty$.

\section{Randomized Learner}
\label{sec:rand}

\subsection{Partial-Feedback Measure Shattering Dimension}

Previous sections focused on deterministic learners. We now turn to learnability with randomized learners. Correspondingly, we introduce a combinatorial dimension that characterizes learnability under randomized algorithms.

\begin{definition}[Partial-Feedback Measure Shattering Dimension (PMSdim)]
Fix $\gamma\in[0,1]$.  ${\rm PMS}_{(d,\gamma)}(\tilde{\H})$ is the largest $q$ such that there exists a depth-$d$
$\X$-valued tree whose edges are indexed by probability measures $\pi_t\in\Pi(\Y)$ and equipped with visible-label annotations
$y^{\rm vis}_t=f^{\rm vis}_t(\pi_1,\ldots,\pi_t)$, with the property that for \emph{every} root-to-leaf sequence
$\Pi_{1:d}=(\pi_1,\ldots,\pi_d)$ there is a subset $F_{\Pi_{1:d}}\in \tilde{\H}$ satisfying:
(i) $f^{\rm vis}_t(\Pi_{1:t})\in F_{\Pi_{1:d}}(\T(\Pi_{1:t-1}))\in \S(\Y)$ for all $t\in[d]$ (consistency with the revealed labels), and
(ii) on at least $q$ rounds, $\pi_t\!\big(F_{\Pi_{1:d}}(\T(\Pi_{1:t-1}))\big)\le 1-\gamma$
(for $\gamma=0$, replace $\le 1-\gamma$ by $<1$).
\end{definition}

\noindent The PMSdim can be understood as a variant of PFLdim (Definition~\ref{def:pfldim}), where the edge labels are replaced with distributions, to be aligned with randomized learners.
The prefix version, Prefix Partial-Feedback Measure Shattering Dimension (PPMSdim, Definition~\ref{def:ppms_inline}), is defined similarly to PPFLdim (Definition~\ref{def:ppfldim}). It helps in proving the bound with randomized algorithms as a proof device.

\begin{theorem}[Regret of Randomized Algorithms]
\label{thm:regretofrand}
    For any $\X$, $\Y$, and $\H$, for all randomized algorithms $\A$, the minimax regret of any $T$-round game satisfies that
    \begin{equation}
    \label{eq:4}
        \sup_{\gamma\in (0,1]} \gamma {\rm PMS}_{(T,\gamma)}(\tilde{\H}) \le \inf_{\A} R_{\A}(T,\H) \le  C \inf_{\gamma\in (0,1]} \left\{ \gamma T + \int^{1}_{\gamma} {\rm PMS}_{(T,\eta)}(\tilde{\H}) \dd\eta \right\},
    \end{equation}
    where $C$ is some universal positive constant. The bounds can be tight up to some constant factors.
\end{theorem}
The proof is mainly motivated by~\cite{onlinelearningsetvaluefeedback} and \cite{daskalakis2022fast}. The difference is that~\cite{onlinelearningsetvaluefeedback}'s argument relies on the property of the descending of the SLdim. In contrast, our PFLdim and PPFLdim do not share this property but can only maintain an upper bound. To tackle this, we change the associated part of proof to a counting argument.

Also notice that Theorem~\ref{thm:regretofrand} presents an integral upper bound. To enhance the interpretability and utilizability, we show a sufficient growth speed of PMSdim to guarantee the learnability.

\begin{corollary}[Logarithmic regret from polynomial ${\rm PMS}$ decay]
\label{cor:log-regret-from-pms}
Suppose there exist constants $C>0$ and $p\in(0,1)$ such that for all $\eta\in(0,1]$, $
{\rm PMS}_{(T,\eta)}(\tilde{\H}) \le C \log T \cdot \eta^{-p} + C \log T$.
Then $\H$ is learnable with randomized algorithms, and $\inf_{\A} R_{\A}(T,\H) = O(\log T)$.
\end{corollary}
\begin{remark*}
Under the same bound, if $p=1$ then optimizing \eqref{eq:4} gives
$\inf_{\A} R_{\A}(T,\H)=O((\log T)^2)$, while if $p>1$ it yields
$\inf_{\A} R_{\A}(T,\H)=O\!\left(T^{1-\frac{1}{p}}(\log T)^{\frac{1}{p}}\right)$.
\end{remark*}


\subsection{Relation between Randomized and Deterministic Settings}

It is immediate that the learnability under deterministic learners implies the learnability under randomized learners, as deterministic learners are a special case of the randomized learners. The key question is when the learnability is identical in these two settings, and when it is not. A structural property---the finiteness of Helly number~\citep{Helly0}---has been proven sufficient for the learnability consistency in the set-valued setting~\citep{onlinelearningsetvaluefeedback}, which has also been used for characterizing proper learning~\citep{helly1,hanneke2021proper}. 

\begin{definition}[Helly Number]
    The Helly number of the collection of sets, $\mathscr{H}(\mathcal{C})$ is the minimum $p$ such that the following condition is met: For any non-empty subcollection $\mathcal{S}$ of $\mathcal{C}$ with $\bigcap_{S\in \mathcal{S}} S = \emptyset$, there exists a non-empty subcollection $\mathcal{S}^{\prime} \subseteq \mathcal{S}$ containing at most $p$ sets, such that $\bigcap_{S^{\prime}\in \mathcal{S}^{\prime}} S^{\prime} = \emptyset$. If such $p$ does not exist, we say that $\mathscr{H}(\mathcal{C})=\infty$.
\end{definition}
Interestingly, we notice that the finiteness of Helly number is not necessary for the learnability consistency in the partial-feedback setting. To prove it, we show that a nested inclusion property also suffices to obtain the consistency. Meanwhile, there exists a $\S(\Y)$ with infinite Helly number, but with the nested inclusion property, making the finiteness of Helly number unnecessary.

\begin{theorem}
\label{thm:sufficientnotnecessary}
    Let $\H$ be a hypothesis class, and $\S(\Y)\subseteq \sigma(\Y)$.  If there exists a finite positive integer $p$, and for every subcollection $\mathcal{C}\in\S(\Y)$ with $\bigcap_{S \in \mathcal{C}} S = \emptyset$, either one of the following two conditions is satisfied: (1) there exists a countable sequence $\{S_i\}^{\infty}_{i=1} \subseteq \mathcal{C}$ with $S_1\supseteq S_2\supseteq \dots$, satisfying $\bigcap^{\infty}_{i=1} S_i = \emptyset$; (2) there exists a subcollection $\S \subseteq \mathcal{C}$ of size $\abs{\S} = p$, satisfying $\bigcap_{S\in \S} S = \emptyset$, then ${\rm PMS}_{(d,\gamma)}(\tilde{\H}) = {\rm PFL}_d(\tilde{\H})$ holds for all $\gamma\in [0,\frac{1}{p}]$. Additionally, if all the $\mathcal{C}$ with $\bigcap_{S \in \mathcal{C}} S = \emptyset$ satisfy (1), then ${\rm PMS}_{(d,\gamma)}(\tilde{\H}) = {\rm PFL}_d(\tilde{\H})$ holds for all $\gamma\in [0,1)$. Meanwhile, there exists an $\S(\Y)$ of which the Helly number is infinity, but the above condition is satisfied.
\end{theorem}

\begin{proofsketch} The proof consists of showing ${\rm PMS}_{(d,\gamma)}(\tilde{\H}) \ge {\rm PFL}_{d}(\tilde{\H})$ and ${\rm PMS}_{(d,\gamma)}(\tilde{\H}) \le {\rm PFL}_{d}(\tilde{\H})$ under the conditions stated in the theorem. Here we show the sketch for proving ${\rm PMS}_{(d,\gamma)}(\tilde{\H}) \ge {\rm PFL}_{d}(\tilde{\H})$. The aim is to construct a tree $\T^{\prime}$ of depth $d$ that is $\gamma$-measure ${\rm PFL}_{d}(\tilde{\H})$-shattered by $\tilde{\H}$ where $\gamma \in [0,\frac{1}{p}]$ or $\gamma \in [0,1)$, depending on the conditions stated in the theorem. We construct the edges/reveals of each node $\T^{\prime}(\Pi_{1:i})$ in $\T^{\prime}$ based on an associated node $\T(\y_{1:i})$ in $\T$. For each $y$, we add all the $S_y$ with $y\notin S_y$ into a collection, $\mathbf{S}_y$. From each of $\mathbf{S}_y$, we select an arbitrary set $S_y\in \mathbf{S}_y$ to form another collection $\S$. It is not hard to see that the intersection of all sets in $\S$ is $\emptyset$. Thus, take $\S$ as $\mathcal C$ in the theorem statement, and either one of (1) or (2) must hold. If (1) holds, then by continuity from above, for every measure $\pi \in \Pi(\Y)$ and $0 < \gamma < 1$, there exists a set $S\in\mathcal{S}$, satisfying that $\pi(S)< 1 - \gamma$. If (2) holds, then by a union bound argument, we have that for every probability measure $\pi$, there must exist an $S_y\in \S$, such that $\pi(S_y)\le 1 - \frac{1}{p}$. We can further assert that for every measure $\pi$ and $0<\gamma<1$, there exists a $y$, such that for any $S_y\in \S(\Y)$ with $y\notin S_y$, we have that either $\pi(S_y)<1-\gamma$ or $\pi(S_y) \le 1-\frac{1}{p}$ holds. We prove by contradiction. Assuming it is false, then there exists a measure $\pi$ and $0<\gamma<1$, for all $y$, there exists a $S_y\in \S(\Y)$ with $y\notin S_y$, such that $\pi(S_y) > 1 - \frac{1}{p}$ and $\pi(S_y) \ge 1-\gamma$ both hold. We can add such $S_y$ for each $y$ to form a new collection $\S$. However, since we have proved that for such $\S$ with empty intersection of all sets,  either $\pi(S_y) \le 1 - \frac{1}{p}$ or $\pi(S_y) < 1-\gamma$ holds, we reach a contradiction. Therefore, we can let the revealed label of edge of $\T^{\prime}(\Pi_{1:i})$ indexed by each $\pi$ be the revealed label of edge indexed by corresponding $y$ of which all $\pi(S_y)$ are small. The ground-truth functions are also adapted from the associated path. From this strategy, we construct a tree $\T^{\prime}$ of depth $d$ that is $\gamma$-measure ${\rm PFL}_{d}(\tilde{\H})$-shattered by $\tilde{\H}$ where $\gamma \in [0,\frac{1}{p}]$ or $\gamma \in [0,1)$, depending on the conditions stated in the theorem.
\end{proofsketch}
\noindent According to Theorem~\ref{thm:dldbp}, the equality between PMSdim and PFLdim in the above theorem implies that $\inf_{\A} R_{\rm rand\ \A}(T,\H) = \Theta( {\rm PFL}_T(\tilde{\H}) ) = {\rm inf}_{\rm Det\ \A} R_{\A}(T,\H)$, thus guarantees the learnability consistency.

\noindent\textbf{Solving open problem raised by \cite{onlinelearningsetvaluefeedback}.}
\cite{onlinelearningsetvaluefeedback} have left an open problem on \emph{whether the finiteness of Helly number is necessary for the learnability consistency in the set-valued setting}. We notice that the above proof can be easily adapted to the set-valued setting, implying that the finiteness of Helly number is not necessary for the learnability consistency in the set-valued setting, answering their open problem.

\begin{corollary}
    Let $\H$ be a hypothesis class, and $\S(\Y)\subseteq \sigma(\Y)$. If $\S(\Y)$ satisfies one of conditions (1) or (2) in Theorem~\ref{thm:sufficientnotnecessary}, then ${\rm MS}_{\gamma}(\H) = {\rm SL}(\H)$ holds for all $\gamma\in [0,\frac{1}{p}]$. Additionally, if all the $\mathcal{C}$ with $\bigcap_{S \in \mathcal{C}} S = \emptyset$ satisfy (1), then ${\rm MS}_{\gamma}(\H) = {\rm SL}_T(\H)$ holds for all $\gamma\in [0,1)$. Meanwhile, there exists a $\S(\Y)$ of which the Helly number is infinity, but the above condition is satisfied. 
\end{corollary}

\subsection{Separation between Oblivious and Public Settings}

So far, we let the adversary select $x_t$ and the learner predict $\pi_t\in\Pi(\Y)$ based on the game history including previous output measures $\Pi_{1:i-1} = (\pi_1,\dots, \pi_{t-1})$, but not the predicted labels $\hat{\y}_{1:i-1} = \{ \hat{y}_1,\dots, \hat{y}_{i-1}\}$ drawn on $\Pi_{1:i-1}$ (unless $\pi = \delta_y$). This information access control is aligned with~\cite{onlinelearningsetvaluefeedback}. We call this setting the \emph{oblivious} setting, and the one with knowing the predicted labels the \emph{public} setting. Notice that our ``oblivious'' terminology does not mean the standard oblivious adversary in online learning~\citep{originaloblivious1,originaloblivious2,originaloblivious3}; it only refers to the fact that past realized labels $\hat \y_{1:t-1}$ are hidden from the adversary when the learner is randomized. We argue that in certain scenarios, the public setting is more reasonable, as the adversary should be able to read the learners' predictions in previous rounds, even if it can be trapped by the randomized prediction strategy in the current round. In contrast to the formulation of minimizing regret in the oblivious setting in Equation~\eqref{eq:regretrewritten}, the one of the public setting is written as
\begin{equation}
\label{eq:pubregret}
\begin{aligned}
R^{\rm pub}(T,\H) &= \sup_{x_1}\inf_{\hat\pi_1}\sup_{y_1^{\rm vis}} \underset{\hat y_{1}\sim \hat \pi_1}{\E} 
\sup_{x_2}\inf_{\hat\pi_2}\sup_{y_2^{\rm vis}} \underset{\hat y_{2}\sim \hat \pi_2}{\E}  \cdots \sup_{x_T}\inf_{\hat\pi_T}\sup_{y_T^{\rm vis}} \underset{\hat y_{T}\sim \hat \pi_T}{\E}  \\ &\qquad \sup_{\substack{ \exists\,F^\star\in\tilde{\H}\ \text{s.t.}\ S_i=F^\star(x_i),\ \forall i\in[T]\\ S_i\in \S(\Y),\ y_i^{\rm vis}\in S_i,\ \forall i\in[T] }} \sum_{t=1}^T \1\{\hat y_t\notin S_t\}.
\end{aligned}
\end{equation}
This formula implies the order of operations in this new model: 
At each round $t$, the adversary picks $x_t$; the learner outputs $\hat \pi_t$; the adversary reveals $y^{\rm vis}_t$; after which the realized prediction $\hat y_t$ is sampled from $\hat \pi_t$. At the final round, the adversary reveals the ground-truth function set, and calculate the total loss that the learner shall incur. 
An illustration of this setting is shown in Table~\ref{tab:settingpublic}. We show that, while these two settings are identical in terms of learnability in the set-valued setting~\citep{onlinelearningsetvaluefeedback}, the learnability is separable in our partial-feedback setting. 

\begin{theorem}[Oblivious=Public for Set-valued Setting]
\label{thm:o=p}
    Fix $\X$, $\Y$, $\S(\Y)$, and $\H$. The existence-realizable online learnability with set-valued feedback and randomized algorithms of $\H$ is identical under oblivious and public settings.
\end{theorem}


\noindent However, the above reasoning cannot be extended to the partial-feedback setting. The following statement shows a separation in learnability.

\begin{theorem}[Oblivious$\neq$Public for Partial-Feedback Learning]
\label{thm:oneqpforpartial}
    There exists a choice of $\X$, $\Y$, $\S(\Y)$, and $\H$ such that $\H$ is partial-feedback online learnable under randomized algorithms, set-realizable, and oblivious settings, but unlearnable in the public setting. 
\end{theorem}
\begin{proofsketch}
Let $\X=\Y=\N_+$, $\H=\{f_{i_1,i_2,\dots} (x) = i_x \ \mid\ i_1,i_2,\dots \in \N_{+}  \}$, and $\S(\Y) = \{S_y=\Y\setminus\{y\} \ \mid\ y\in\Y   \}$. \emph{Oblivious:} output $\mathrm{Unif}([T])$ each round; since each
$S_t$ excludes exactly one label, per-round error $\le 1/T$, giving $O(1)$ expected total loss.
\emph{Public:} in round $t$ pick $x_t=t$ and reveal $y_t^{\rm vis}$ with $\pi_t(y_t^{\rm vis})\le 1-k$
(possible as $\Y$ is infinite). After $\hat \y_{1:T}$ are realized, choose $S_t\in\S(\Y)$ consistent
with $y_t^{\rm vis}$ so that $\pi_t(S_t)\le 1-k$ (adversary can always choose some $S_t$ such that $\hat y_t\notin S_t$, if $\hat y_t\neq y_t^{\rm vis}$), yielding per-round loss $\ge k$ and thus $\Omega(T)$.
\end{proofsketch}

With public assumption, what distinguishes the situation in the partial-feedback setting from that in the set-valued setting is that in the former setting, the ground-truth function set, and therefore the ground-truth label sets can be determined after the predicted labels are drawn. Thus, the benefits from the randomization are largely eliminated. However, it does not directly imply that there is no separation between randomized and deterministic learners, as the adversary fixes the revealed label before the actual predicted label is drawn in each round. The revealed label limits the adversary's choosing ground-truth function set. We leave it as an open problem to see if there is any learnability separation between randomized and deterministic learners under public setting.

\section{Agnostic and Existence-Realizable Settings}
\label{sec:arrs}
In this section, we show that without further noise control, the existence-realizable setting is extremely hard to learn, let alone the agnostic setting. Concretely, we show that there exists a two-element hypothesis class that can force any learner to incur linear regret.

\begin{theorem}
     Under existence-realizable and oblivious settings, there exists a choice of $\X$, $\Y$, $\S(\Y)$, and $\H$, such that $\H$ is not partial-feedback online learnable even if $\abs{\H}=2$ with randomized learners.
\end{theorem}
\begin{proof}
    We provide such an example.
\begin{example}
    Let $\X:=\N_+$, $\Y:=\{0,1\}$, $\S(\Y) := \PP(\Y)$, and $\H:=\{f_0\equiv 0, \ f_1\equiv 1 \}$. Assume that the game is played $T$ rounds. Let the adversary show instances following the order $x_1=1$, $x_2=2$, $\dots$, $x_T=T$. In each round $t$, the adversary reveals the label $y^{\rm vis}_t$ that the learner puts less probability (reveals $y^{\rm vis}_t=1$ if tie). Notice that no matter what algorithm the learner uses, it will predict $0$ with probability $\ge 0.5$ for $q$ instances and $1$ with probability $\ge 0.5$ for $\ge T-q$ instances. The adversary assigns the ground-truth function to be $f_k$ where $k:=\argmin_{i=\{0,1\}} (1-i)q+ (T-q)i$. For those rounds where the learner predicts $1-k$ with higher probability, the ground-truth label sets are set to be $\{k\}$; For those rounds where the learner predicts $k$ with higher probability, the ground-truth label sets are set to be $\{0,1\}$; For those rounds where the learner outputs tied probabilities, if $k=0$, then the ground-truth label sets are set to be $\{0,1\}$, otherwise $\{1\}$. It is not hard to see that the algorithm will incur loss $\ge 0.5$ for at least $\frac{T}{2}$ rounds. Thus, the regret grows linearly in $T$ and it is not online learnable.
\end{example}
\end{proof}

\noindent In classical online learning, the agnostic setting is often interpreted as a noisy regime, where the mismatch between the hypothesis class and the ground truth is attributed to noise or model misspecification~\citep{hanneke2012activized,agnosticisnoise}. Existence realizability shares a similar flavor by allowing some labels to fall outside the predictions of a fixed hypothesis. In contrast, the realizability in the multiclass protocol is a special case of set realizability, since each round outputs a single label from a fixed predictor. While such “noise” can be tolerated under set-valued feedback, partial feedback is strictly weaker: the adversary may leverage such noise to make even very simple classes unlearnable. This highlights the need for noise-sensitive complexity measures that jointly capture hypothesis complexity and the severity of non-realizability under partial feedback. We leave the identification of an appropriate combinatorial dimension and its learnability implications as an open direction.

\section{Conclusion}

This paper presents a systematic study of partial-feedback online learning, a new online learning protocol in which each instance may admit multiple acceptable labels, while the learner observes only a single witness label in each round and never immediately knows whether its own prediction is correct. To address the key difficulty that the classical version space does not naturally extend to this setting, the paper introduces the notion of collection version space, which lifts consistency from individual hypotheses to sets of hypotheses and restores a useful monotone structure for online analysis. Building on this idea, the paper provides a tight characterization of learnability in the set-realizable regime by proposing the PFLdim and the PMSdim, and showing that they exactly characterize the minimax regret of deterministic and randomized learners, respectively. The paper further identifies a nested-inclusion condition under which deterministic and randomized learnability coincide, thereby resolving an open question in prior work. At the same time, it shows a sharp impossibility result beyond set realizability: even when the hypothesis class contains only two elements, the minimax regret can still be linear. Overall, this work establishes a new theoretical foundation for partial-feedback online learning and suggests that future progress will require new complexity measures and techniques that are robust to weaker realizability assumptions and more challenging feedback structures.

\section*{Acknowledgment}

Shihao Shao would like to thank Xinbu Cheng for taking the time to verify the proofs and for always being willing to help with verification when needed.

\bibliographystyle{plainnat}
\bibliography{ref}

@InProceedings{onlinelearningsetvaluefeedback,
  title = 	 {Online Learning with Set-valued Feedback},
  author =       {Raman, Vinod and Subedi, Unique and Tewari, Ambuj},
  booktitle = 	 {Proceedings of Thirty Seventh Conference on Learning Theory},
  pages = 	 {4381--4412},
  year = 	 {2024},
  editor = 	 {Agrawal, Shipra and Roth, Aaron},
  volume = 	 {247}
}

@inproceedings{hanneke2023multiclass,
  title={Multiclass online learning and uniform convergence},
  author={Hanneke, Steve and Moran, Shay and Raman, Vinod and Subedi, Unique and Tewari, Ambuj},
  booktitle={The Thirty Sixth Annual Conference on Learning Theory},
  pages={5682--5696},
  year={2023}
}

@InProceedings{hanneke2021proper,
  title = 	 {Online Learning with Simple Predictors and a Combinatorial Characterization of Minimax in 0/1 Games},
  author =       {Hanneke, Steve and Livni, Roi and Moran, Shay},
  booktitle = 	 {The Thirty Fourth Conference on Learning Theory},
  pages = 	 {2289--2314},
  year = 	 {2021}
}

@inproceedings{xu2022onepositive,
title={One Positive Label is Sufficient: Single-Positive Multi-Label Learning with Label Enhancement},
author={Ning Xu and Congyu Qiao and Jiaqi Lv and Xin Geng and Min-Ling Zhang},
booktitle={Advances in Neural Information Processing Systems},
year={2022}

}

@inproceedings{originaloblivious1,
author = {Kleinberg, Robert},
title = {Anytime algorithms for multi-armed bandit problems},
year = {2006},
pages = {928–936},
series = {The 17th Annual ACM-SIAM Symposium on Discrete Algorithms}
}

@inproceedings{originaloblivious2,
author = {Arora, Raman and Dekel, Ofer and Tewari, Ambuj},
title = {Online bandit learning against an adaptive adversary: from regret to policy regret},
year = {2012},
booktitle = {The 29th International Coference on International Conference on Machine Learning},
pages = {1747–1754}
}

@InProceedings{originaloblivious3,
  title = 	 {Contextual Multi-Armed Bandits},
  author = 	 {Lu, Tyler and Pal, David and Pal, Martin},
  booktitle = 	 {The Thirteenth International Conference on Artificial Intelligence and Statistics},
  pages = 	 {485--492},
  year = 	 {2010}
}

@inproceedings{pll2,
author = {Zhang, Min-Ling and Yu, Fei},
title = {Solving the partial label learning problem: an instance-based approach},
year = {2015},
pages = {4048–4054},
series = {The 24th International Conference on Artificial Intelligence}
}

@article{mitchell1982generalization,
  title={Generalization as search},
  author={Mitchell, Tom M},
  journal={Artificial intelligence},
  volume={18},
  number={2},
  pages={203--226},
  year={1982},
  publisher={Elsevier}
}

@article{tian2023partial,
  title={Partial label learning: Taxonomy, analysis and outlook},
  author={Tian, Yingjie and Yu, Xiaotong and Fu, Saiji},
  journal={Neural Networks},
  volume={161},
  pages={708--734},
  year={2023}
}

@article{cour2011partial,
  author  = {Timothee Cour and Ben Sapp and Ben Taskar},
  title   = {Learning from Partial Labels},
  journal = {Journal of Machine Learning Research},
  year    = {2011},
  volume  = {12},
  number  = {42},
  pages   = {1501--1536}
}

@InProceedings{filmus2023randomized,
  title = 	 {Optimal Prediction Using Expert Advice and Randomized Littlestone Dimension},
  author =       {Filmus, Yuval and Hanneke, Steve and Mehalel, Idan and Moran, Shay},
  booktitle = 	 {The Thirty Sixth Conference on Learning Theory},
  pages = 	 {773--836},
  year = 	 {2023}
}

@inproceedings{mitchell1977version,
  title={Version spaces: A candidate elimination approach to rule learning},
  author={Mitchell, Tom M},
  booktitle={Proceedings of the 5th international joint conference on Artificial intelligence-Volume 1},
  pages={305--310},
  year={1977}
}

@inproceedings{littlestone,
author = {Littlestone, Nick},
title = {Learning quickly when irrelevant attributes abound: A new linear-threshold algorithm},
year = {1987},
booktitle = {Proceedings of the 28th Annual Symposium on Foundations of Computer Science},
pages = {68–77}
}

@article{Helly0,
author = {Helly, Ed.},
journal = {Jahresbericht der Deutschen Mathematiker-Vereinigung},
pages = {175-176},
title = {Über Mengen konvexer Körper mit gemeinschaftlichen Punkte.},
volume = {32},
year = {1923},
}

@InProceedings{fdbkgraph,
  title = 	 {Online Learning with Feedback Graphs Without the Graphs},
  author = 	 {Cohen, Alon and Hazan, Tamir and Koren, Tomer},
  booktitle = 	 {The 33rd International Conference on Machine Learning},
  pages = 	 {811--819},
  year = 	 {2016}
}

@inproceedings{sideobservation,
 author = {Mannor, Shie and Shamir, Ohad},
 booktitle = {Advances in Neural Information Processing Systems},
 title = {From Bandits to Experts: On the Value of Side-Observations},
 year = {2011}
}

@article{partialmultilabellearning, title={Partial Multi-Label Learning},  journal={Proceedings of the AAAI Conference on Artificial Intelligence}, author={Xie, Ming-Kun and Huang, Sheng-Jun}, year={2018} }

@InProceedings{helly1,
  title = 	 {Proper Learning, Helly Number, and an Optimal SVM Bound},
  author =       {Bousquet, Olivier and Hanneke, Steve and Moran, Shay and Zhivotovskiy, Nikita},
  booktitle = 	 {The Thirty Third Annual Conference on Learning Theory},
  pages = 	 {582--609},
  year = 	 {2020}
}

@InProceedings{raman2024multiclass,
  title = 	 {Multiclass Online Learnability under Bandit Feedback},
  author =       {Raman, Ananth and Raman, Vinod and Subedi, Unique and Mehalel, Idan and Tewari, Ambuj},
  booktitle = 	 {The Thirty Fifth International Conference on Algorithmic Learning Theory},
  pages = 	 {997--1012},
  year = 	 {2024}
}

@article{bartok2014partialmonitoring,
author = {Bart\'{o}k, G\'{a}bor and Foster, Dean P. and P\'{a}l, D\'{a}vid and Rakhlin, Alexander and Szepesv\'{a}ri, Csaba},
title = {Partial Monitoring—Classification, Regret Bounds, and Algorithms},
journal = {Mathematics of Operations Research},
volume = {39},
number = {4},
pages = {967-997},
year = {2014},
doi = {10.1287/moor.2014.0663}
}

@article{auer2003nonstochastic,
author = {Auer, Peter and Cesa-Bianchi, Nicol\`{o} and Freund, Yoav and Schapire, Robert E.},
title = {The Nonstochastic Multiarmed Bandit Problem},
journal = {SIAM Journal on Computing},
volume = {32},
number = {1},
pages = {48-77},
year = {2002},
doi = {10.1137/S0097539701398375}
}

@InProceedings{Daniely2013price,
  title = 	 {The price of bandit information in multiclass online classification},
  author = 	 {Daniely, Amit and Helbertal, Tom},
  booktitle = 	 {The Twenty Sixth Annual Conference on Learning Theory},
  pages = 	 {93--104},
  year = 	 {2013},
  editor = 	 {Shalev-Shwartz, Shai and Steinwart, Ingo},
}

@inproceedings{daniely2011littlestone,
  title={Multiclass learnability and the erm principle},
  author={Daniely, Amit and Sabato, Sivan and Ben-David, Shai and Shalev-Shwartz, Shai},
  booktitle={The Twenty Forth Annual Conference on Learning Theory},
  pages={207--232},
  year={2011}
}

@article{agnosticisnoise,
author = {Kalai, Adam Tauman and Klivans, Adam R. and Mansour, Yishay and Servedio, Rocco A.},
title = {Agnostically Learning Halfspaces},
journal = {SIAM Journal on Computing},
volume = {37},
number = {6},
pages = {1777-1805},
year = {2008},
doi = {10.1137/060649057},
}

@inproceedings{rakhlin2011online,
  title={Online learning: Beyond regret},
  author={Rakhlin, Alexander and Sridharan, Karthik and Tewari, Ambuj},
  booktitle={The Twenty Forth Annual Conference on Learning Theory},
  pages={559--594},
  year={2011}
}

@article{rakhlin2015sequential,
  author  = {Alexander Rakhlin and Karthik Sridharan and Ambuj Tewari},
  title   = {Online Learning via Sequential Complexities},
  journal = {Journal of Machine Learning Research},
  year    = {2015},
  volume  = {16},
  number  = {6},
  pages   = {155--186},
  url     = {http://jmlr.org/papers/v16/rakhlin15a.html}
}

@inproceedings{bendavid2009agnostic,
  author    = {Ben-David, Shai and P{\'a}l, D{\'a}vid and Shalev-Shwartz, Shai},
  title     = {Agnostic Online Learning},
  booktitle = {The Twenty Second Annual Conference on Learning Theory},
  year      = {2009}
}

@article{hanneke2012activized,
  title={Activized learning: Transforming passive to active with improved label complexity},
  author={Hanneke, Steve},
  journal={The Journal of Machine Learning Research},
  volume={13},
  number={1},
  pages={1469--1587},
  year={2012},
  publisher={JMLR. org}
}

@inproceedings{daskalakis2022fast,
  title={Fast rates for nonparametric online learning: from realizability to learning in games},
  author={Daskalakis, Constantinos and Golowich, Noah},
  booktitle={Proceedings of the 54th Annual ACM SIGACT Symposium on Theory of Computing},
  pages={846--859},
  year={2022}
}

\clearpage

\appendix


\begin{algorithm}[t]
\caption{Collection Version Space Pruning Algorithm}
\label{alg:cvspa}
\SetKwInOut{Initialize}{Initialize}

\Initialize{$\tilde{V}_0=\tilde{\H}$;}

\For{$t=1,\dots,T$}{

  receive $x_t \in \mathcal{X}$;

  \For{$y\in\Y$}{
  \If{$y\in F(x_t)$ holds for every $F\in \tilde{V}_{t-1}$}{
  predict $y$ and quit this for-loop;
  }
  }
  \If{No prediction is made in the above for-loop}{
  predict the $y\in\Y$ that most of $f\in \H$ satisfying that $f(x_t)$ equals $y$;
  }

  receive $y^{\rm vis}_t\in S_t$ \tcp*{$S_t=F(x_t)$ for some $F\in \tilde{V}_{t-1}$}

  \begin{equation}
  \label{eq:collectionversionspaceupdate}
  \tilde{V}_t \gets \{F\in\tilde{V}_{t-1}\ \mid \ y^{\rm vis}_t\in F(x_t) \};\
  \end{equation}

}
\end{algorithm}

\section{Proof of Finite Hypothesis Spaces}
\label{app:finitebound}

\setcounter{theorem}{1}
\begin{theorem}[Learnability with Finite Hypothesis Class (Coarse Bound)]
    Under partial feedback and set-realizability, any finite $\H$ with $|\H|=n$ is online learnable by a deterministic algorithm with
minimax regret at most $\sum_{i=1}^{\lfloor n/2\rfloor}\binom{n}{i}$.
\end{theorem}

\begin{proof}
    The upper bound on the minimax regret is obtained by upper bounding the regret of Algorithm~\ref{alg:cvspa}. Under set realizability, the comparator incurs zero loss, hence the regret equals the learner’s number of mistakes. We first notice that, under the set-realizable assumption, there exists a ground-truth function set $F^{\star}$ such that $F^{\star} \in \tilde{V}_t$ for every $t$. Indeed, $\tilde{V}_0 = \tilde{\mathcal H}$ contains $F^{\star}$ by the set-realizable assumption. Moreover, if $F^{\star} \in \tilde{V}_{t-1}$, then $y_t^{\rm vis} \in S_t = F^{\star}(x_t)$, and hence the update rule in Equation~\eqref{eq:collectionversionspaceupdate} implies that $F^{\star} \in \tilde{V}_t$. Therefore, by induction, $F^{\star} \in \tilde{V}_t$ for all $t$. Now, if there exists some $y \in \mathcal Y$ such that $y \in F(x_t)$ for every $F \in \tilde{V}_{t-1}$ in round $t$, then in particular $y \in F^{\star}(x_t) = S_t$, since $F^{\star} \in \tilde{V}_{t-1}$. Hence, predicting $y$ is correct. 
    
    Therefore, the learner can only make a mistake in rounds where no such $y$ exists, that is, the algorithm outputs $y$ outside the for-loop. To upper bound the regret, it therefore suffices to upper bound the number of such rounds. To this end, define
\begin{equation}
\label{eq:vvisdef}
V_t^{\rm vis} := \{f \in \mathcal H \mid f(x_t) = y_t^{\rm vis}\}.
\end{equation}
The key observation is that if the learner makes a mistake outside the for-loop in two rounds $t$ and $t'$ with $t < t'$, then $V_t^{\rm vis} \neq V_{t'}^{\rm vis}$. Once this is established, it remains to bound the number of distinct sets $V_t^{\rm vis}$, which yields the desired upper bound on the number of mistakes.

    To prove this, it suffices to prove that if $V^{\rm vis}_{t} = V^{\rm vis}_{t^{\prime}}$, then we have that $y^{\rm vis}_{t^{\prime}} \in F(x_{t^{\prime}})$ for all $F\in \tilde{V}_{t^{\prime}-1}$, which implies that the prediction at round $t^{\prime}$ will be made in the for-loop and will be correct.
    To prove this, we show that Equation~\eqref{eq:collectionversionspaceupdate} is equivalent to
    \begin{equation}
    \label{eq:vtintersect}
        \tilde V^{\prime}_t := \left\{ F \in \tilde V_{t-1} \mid F\ \bigcap\ V^{\rm vis}_t \neq \emptyset \right\},
    \end{equation}
    in other words, we aim to show that $\tilde V_t=\tilde V^{\prime}_t$. It suffices to prove both inclusions: a function set is in $\tilde V_t$ implies that it is also in $\tilde V^{\prime}_t$, and \textit{vice versa}. We begin with the first direction, and without loss of generality consider a function set $F\in \tilde V_t$. Equation~\ref{eq:collectionversionspaceupdate} implies that $y^{\rm vis}_t \in F(x_t)$, which means that $y^{\rm vis}_t = f(x_t) \in F(x_t)$ for some $f\in F$. According to Equation~\eqref{eq:vvisdef}, this implies that $f\in V^{\rm vis}_t$, which further gives that $F \cap V^{\rm vis}_t \supseteq \{f\} \neq \emptyset$. This fits the definition of $\tilde V^{\prime}_t$ (Equation~\eqref{eq:vtintersect}) and thus $F$ is in $\tilde V^{\prime}_t$. Next, we prove the reverse direction with an arbitrary $G\in \tilde V^{\prime}_t$. According to Equation~\eqref{eq:vtintersect}, we have that $G\ \cap\ V^{\rm vis}_t =: Q \neq \emptyset$. By Equation~\eqref{eq:vvisdef}, there exists at least one function $g\in G$ such that $g(x_t) = y^{\rm vis}_t$, which implies that $y^{\rm vis}_t = g(x_t) \in G(x_t)$. Thus, we have that $G\in \tilde V_t$. Combining the two directions, we prove that $\tilde V_t = \tilde V^{\prime}_t$. This means that the update rule is equivalent to Equation~\eqref{eq:vtintersect}. Next, we show that if $V^{\rm vis}_{t} = V^{\rm vis}_{t^{\prime}}$ for some $t<t'$, then the learner cannot make a mistake outside the for-loop in round $t'$. Since the collection version space is monotone decreasing over time, we have
\[
\tilde V_{t'-1} \subseteq \tilde V_t .
\]
Now take any $F \in \tilde V_{t'-1}$. Then $F \in \tilde V_t$, and since we have already proved that the update rule in Equation~\eqref{eq:collectionversionspaceupdate} is equivalent to
Equation~\eqref{eq:vtintersect},
it follows that
\[
F\cap V_t^{\rm vis}\neq\emptyset.
\]
If $V_t^{\rm vis}=V_{t'}^{\rm vis}$, then
\[
F\cap V_{t'}^{\rm vis}\neq\emptyset.
\]
Since $F\in \tilde V_{t'-1}$ was arbitrary, this shows that every $F\in \tilde V_{t'-1}$ satisfies
\[
F\cap V_{t'}^{\rm vis}\neq\emptyset.
\]
By the definition of $V_{t'}^{\rm vis}$, if $F\cap V_{t'}^{\rm vis}\neq\emptyset$, then there exists some $f\in F$ such that $f(x_{t'})=y_{t'}^{\rm vis}$, and hence
\[
y_{t'}^{\rm vis}\in F(x_{t'}).
\]
Therefore, $y_{t'}^{\rm vis}\in F(x_{t'})$ for every $F\in \tilde V_{t'-1}$. This means that in round $t'$, the prediction is made inside the for-loop, and since $F^\star\in \tilde V_{t'-1}$, the predicted label must belong to $F^\star(x_{t'})=S_{t'}$. Hence the learner is correct in round $t'$. We conclude that if the learner makes a mistake outside the for-loop in two rounds $t$ and $t'$ with $t<t'$, then necessarily
\[
V_t^{\rm vis}\neq V_{t'}^{\rm vis}.
\]
    
    Finally, consider any round $t$ in which the learner makes a prediction outside the for-loop and makes a mistake. In this case, the algorithm predicts a label $\hat y_t$ by taking a majority vote over the hypotheses in $\mathcal H$ at $x_t$. Since the learner is wrong, we must have $\hat y_t \neq y_t^{\rm vis}$. Therefore, the number of hypotheses $f\in\mathcal H$ such that $f(x_t)=y_t^{\rm vis}$ cannot exceed the number of hypotheses predicting $\hat y_t$. In particular,
\[
|V_t^{\rm vis}|
=
\left|\left\{f\in\mathcal H \mid f(x_t)=y_t^{\rm vis}\right\}\right|
\le \left\lfloor \frac{n}{2}\right\rfloor.
\]
Combining this with the argument above, each mistaken round outside the for-loop corresponds to a distinct subset $V_t^{\rm vis}\subseteq\mathcal H$ of size at most $\lfloor n/2\rfloor$, where $\abs{V_t^{\rm vis}} > 0$ as by the set-realizable assumption, there must exist an $F^{\star}\in \tilde V_{t-1}$ such that $y^{\rm vis}_t \in F^{\star}(x_t)$, so there exists an $f\in F^{\star} \subseteq \H$, such that $f(x_t) = y^{\rm vis}_t$ and therefore $\abs{V_t^{\rm vis}} > 0$. Hence, the total number of mistakes is at most the number of such subsets, namely
\[
\sum_{i=1}^{\lfloor n/2\rfloor}\binom{n}{i}.
\]
This completes the proof.
\end{proof}

\section{Proof of Learnability with Deterministic Learners}

\setcounter{theorem}{5}
\begin{theorem}[Deterministic Learnability Determined by PFLdim]
\label{thm:dldbp}
Fix a subset collection $\S(\Y)\subseteq \sigma(\Y)$ and a hypothesis class $\H\subseteq \Y^{\X}$. Under the set-realizable setting, it holds that
\[
R^{\star}_{\rm det}(T,\H)={\rm PFL}_T(\tilde{\H})
\qquad\text{for all } T\in\N_+.
\]
\end{theorem}

\begin{proof}
The proof consists of matching upper and lower bounds.

\paragraph{Upper bound}
Let $\A$ denote Algorithm~\ref{alg:dpfla} (DPFLA).
We show that for every horizon $T$, the regret of $\A$ satisfies
\[
R^{\star}_{\rm det}(T,\H)=\inf_{{\rm det\ } \A^{\prime} } R_{\A^{\prime}} (T,\H) \le R_{\A}(T,\H)\;\le\;{\rm PFL}_T(\tilde{\H}).
\]
The argument proceeds via the prefix-aware auxiliary dimension ${\rm PPFLdim}$.

Let $\x_{1:T}$, $\y^{\rm vis}_{1:T}$ be an arbitrary sequence given by the adversary against $\A$. Based on the definition of PPFLdim, we have that
\begin{equation}
\label{eq:ppflmax}
\begin{aligned}
    {\rm PPFL}_0(\x_{1:T},\hat\y_{1:T},\tilde{V}_T) &= \max_{ \substack{F\in\tilde V_T\\ F(x_i)\in \S(\Y), \ \forall i\in [T]}   } \sum^T_{j=1} \1 \left\{ \hat\y_j \notin F(x_j) \right\}
    \\ &=\max_{\substack{F\in\tilde\H \\ {\rm s.t.\ } y^{\vis}_i\in F(x_i)\in \S(\Y), \ \forall i\in [T]   }} \sum^T_{j=1} \1 \left\{ \hat\y_j \notin F(x_j) \right\},
\end{aligned}
\end{equation}
where $\hat\y_j$ is output by $\A$ based on $\x_{1:j}$ and $\y^{\rm vis}_{1:j-1}$; $\tilde V_T$ is based on $\x_{1:T}$ and $\y^{\rm vis}_{1:T}$ according to DPFLA. Also notice that
\begin{equation*}
    \begin{aligned}
        R^{\star}_{\rm det}(T,\H)&\le R_{\A} (T,\H) \\
        &= \max_{\x_{1:T},\y^{\vis}_{1:T}} \max_{\substack{F\in\tilde\H \\ {\rm s.t.\ } y^{\vis}_i\in F(x_i)\in \S(\Y), \ \forall i\in [T]   }} \sum^T_{j=1} \1 \left\{ \hat\y_j \notin F(x_j) \right\}\\
        &\overset{Eq.~\eqref{eq:ppflmax}}{\le} \max_{\x_{1:T},\y^{\vis}_{1:T}} {\rm PPFL}_0(\x_{1:T},\hat\y_{1:T},\tilde{V}_T)
    \end{aligned}
\end{equation*}

Thus, it suffices to show that for all sequences $\x_{1:T}$ and $\y^{\rm vis}_{1:T}$, the PPFL potential along the trajectory of DPFLA never exceeds the
initial PFL value:
\begin{equation}
\label{eq:ppfl_le_pfl}
{\rm PPFL}_{T-t}\big(\x_{1:t},\hat{\y}_{1:t},\tilde V_t\big)\;\overset{!}{\le}\;{\rm PFL}_T\big(\tilde V_0\big)
\;=\;{\rm PFL}_T\big(\tilde{\H}\big)\qquad\forall\, t\in\{0,1,\ldots,T\}.
\end{equation}

We prove Inequality~\eqref{eq:ppfl_le_pfl} by induction on $t$.
For $t=0$, the prefix is empty and ${\rm PPFL}_{T}(\x_{1:0},\hat{\y}_{1:0},\tilde V_0)={\rm PFL}_T(\tilde V_0)$ by
definition.
Assume Inequality~\eqref{eq:ppfl_le_pfl} holds at time $t-1$, and consider round $t$.
Suppose for contradiction that
\[
{\rm PPFL}_{T-t}\big(\x_{1:t},\hat{\y}_{1:t},\tilde V_t\big)\;>\;{\rm PFL}_T\big(\tilde V_0\big) \;\ge\; {\rm PPFL}_{T-t+1}\big(\x_{1:t-1},\hat{\y}_{1:t-1},\tilde V_{t-1}\big).
\]
We will show that this assumption enables us to construct a tree of depth $T-t+1$, which is $(\x_{1:t-1},\hat\y_{1:t-1} )$-prefix-$q$-shattered with respect to $\tilde V_{t-1}$ for
\begin{equation*}
    q \;\ge\; {\rm PPFL}_{T-t+1}\big(\x_{1:t-1},\hat{\y}_{1:t-1},\tilde V_{t-1}\big) + 1,
\end{equation*}
which according to the definition of PPFLdim gives us that 
\begin{equation}
\label{eq:desiredcontradiction}
    {\rm PPFL}_{T-t+1}\big(\x_{1:t-1},\hat{\y}_{1:t-1},\tilde V_{t-1}\big)\ge {\rm PPFL}_{T-t+1}\big(\x_{1:t-1},\hat{\y}_{1:t-1},\tilde V_{t-1}\big) + 1,
\end{equation}
and therefore a contradiction.

Now, let us get down to the construction. By the prediction strategy (Equation~\eqref{eq:choiceofdpfla}) of DPFLA, we have that

\begin{equation}
\label{eq:ppfl_jump}
\max_{\overline y^{\rm vis}_t\in Q_t}\;
{\rm PPFL}_{T-t}\!\left(\x_{1:t-1}\sqcup(x_t),\ \hat{\y}_{1:t-1}\sqcup(\overline y_t),\ \tilde V^{(x_t,\overline y^{\rm vis}_t)}_{t-1}\right)
\;\ge\;
{\rm PPFL}_{T-(t-1)}\!\left(\x_{1:t-1},\hat{\y}_{1:t-1},\tilde V_{t-1}\right)+1
\end{equation}
holds for each $\overline y_t\in\Y$. That is, for each $\overline y_t$, there exists a $\overline y^{\rm vis}_t$ and an $\X$-valued, $\Y$-ary full tree $\T_{\overline y_t}$ of depth $T-t$
that is $(\x_{1:t},\hat\y_{1:t-1}\sqcup (\overline y_t) )$-prefix-$q$-shattered with respect to $\tilde V^{(x_t,\overline y^{\rm vis}_t)}_{t-1}$ where 
\begin{equation}
\label{eq:qgreater}
    q \;\ge\;
{\rm PPFL}_{T-(t-1)}\!\left(\x_{1:t-1},\hat{\y}_{1:t-1},\tilde V_{t-1}\right)+1.
\end{equation}
That means there exists a sequence of annotation functions $(f^{\rm vis}_{\overline y_t,t+1},\ldots,f^{\rm vis}_{\overline y_t,T})$ such that for every continuation path $\y_{t+1:T}\in\Y^{T-t}$, there exists $F_{\uu}\in\tilde V^{(x_t,\overline y^{\rm vis}_t)}_{t-1}$, where $\uu := {\hat \y_{1:t-1} \sqcup (\overline y_t) \sqcup \y_{t+1:T}  }$ satisfying:

\begin{enumerate}
\item[(i)] For all $k\in\{t+1,\ldots,T\}$,
\[
f^{\rm vis}_{\overline y_t,k}(\hat\y_{1:t-1} \sqcup (\overline y_t) \sqcup \y_{t+1:k} ) \in F_{\uu}\!\big(\T_{\overline y_t}( \y_{t+1:k-1})\big)\in \S(\Y).
\]

\item[(ii)] The total number of indices on which the predicted label is not contained in the corresponding valid-label set, counting both the fixed history and the continuation, is at least $q$, i.e.,
\begin{equation*}
\begin{aligned}
&\Big|\{k\in[t-1]: \hat y_k\notin F_{\uu}(x_k)\}\Big|
\;+\; \Big| \overline{y}_t \notin F_{\uu}(x_t) \Big| \\ & \;+\;
\Big|\{j\in\{t+1,\ldots,T\}: y_j\notin F_{\uu  }\!\big(\T_{\overline y_t}(\y_{t+1:j-1})\big)\}\Big|
\;\ge\; q.
\end{aligned}
\end{equation*}
\end{enumerate}

Now, we are ready to construct our desired tree. We construct an $\X$-valued, $\Y$-ary full tree $\T$ of depth $T-t+1$ in the following way: First, we take $x_t$ as the root node, i.e., $\T(\emptyset) := x_t$. Then, for each of the outreaching edges $\overline y_t \in \Y$ ($\overline y_t$ was used to traverse all elements in $\Y$, so it is okay to use notation $\overline y_t$ instead of a plain $y$), we append the tree $\T_{\overline y_t}$ right after the edge $\overline y_t$. That is, $\T( (\overline y_t) \sqcup \y_{t+1:k}) := \T_{\overline y_t}(\y_{t+1:k})$. Next, we show that conditions (i) and (ii) both hold for $\T$ with prefix $(\x_{t-1},\hat \y_{t-1})$. 
Let $f^{\rm vis}_t(\hat \y_{1:t-1} \sqcup (\overline y_t)) := \overline y^{\rm vis}_t$ and $f^{\rm vis}_{k}(\hat \y_{1:t-1} \sqcup (\overline y_t) \sqcup \y_{t+1:k}) := f^{\rm vis}_{\overline y_t,k}(\hat \y_{1:t-1} \sqcup (\overline y_t) \sqcup \y_{t+1:k})$ for $k\ge t+1$.
Without loss of generality, we consider a continuation path $(\overline y_t) \sqcup \y_{t+1:T}\in \Y^{T-t+1}$. Note that 
\begin{equation}
\label{eq:felementof}
F_{\uu} \in \tilde V^{(x_t,\overline y^{\rm vis}_t)}_{t-1} \subseteq \tilde V_{t-1}, 
\end{equation}
and we have that

\begin{enumerate}
\item[(1a)] For all $k\in\{t+1,\ldots,T\}$,
\[
f^{\rm vis}_{k}(\hat\y_{1:t-1} \sqcup (\overline y_t) \sqcup \y_{t+1:k} ) \in F_{\uu}\!\big(\T_{\overline y_t}( \y_{t+1:k-1})\big) = F_{\uu}\!\big( \T( (\overline y_t) \sqcup \y_{t+1:k-1} ) \big) \in \S(\Y).
\]

\item[(1b)] For $k=t$,

\[
f^{\rm vis}_{k}(\hat\y_{1:t-1} \sqcup (\overline y_t) ) = \overline y^{\rm vis}_t \overset{(*)}{\in} F_{\uu}(x_t) = F_{\uu}\!\big(\T(\emptyset) \big) \in \S(\Y),
\]
where element-of relation (*) is from Equation~\eqref{eq:felementof}. Conditions (1a) and (1b) together complete the consistency requirement (condition~(i)).

\item[(2a)] For $k = t$, if $\overline y_k \notin F_{\uu}(x_k)$, then we have that $\overline y_k \notin  F_{\uu}(\T(\emptyset))$, as $\T(\emptyset) = x_k$, and \textit{vice versa}.

\item[(2b)] For $k \in \{t+1,\cdots,T\}$, if $y_k \notin F_{\uu}(\T_{\overline y_t}(\y_{t+1:k-1}))$, then we have that $y_k \notin F_{\uu}(\T( (\overline{y}_t) \sqcup \y_{t+1:k-1}))$, as $\T_{\overline y_t}(\y_{t+1:k-1}) = \T( (\overline{y}_t) \sqcup \y_{t+1:k-1})$, and \textit{vice versa}.

Put (2a) and (2b) together, and we have that 

\begin{equation*}
\begin{aligned}
&\Big|\{k\in[t-1]: \hat y_k\notin F_{\uu}(x_k)\}\Big|
\;+\; \Big| \overline{y}_t \notin F_{\uu}(\T(\emptyset)) \Big| \\ & \;+\;
\Big|\{j\in\{t+1,\ldots,T\}: y_j\notin F_{\uu  }\!\big(\T((\overline y_t)\sqcup\y_{t+1:j-1})\big)\}\Big|\\
&=\Big|\{k\in[t-1]: \hat y_k\notin F_{\uu}(x_k)\}\Big|
\;+\; \Big| \overline{y}_t \notin F_{\uu}(x_t) \Big| \\ & \;+\;
\Big|\{j\in\{t+1,\ldots,T\}: y_j\notin F_{\uu  }\!\big(\T_{\overline y_t}(\y_{t+1:j-1})\big)\}\Big|
\;\ge\; q.
\end{aligned}
\end{equation*}
This meets the non-membership requirement (condition (ii)).
\end{enumerate}
The above properties, consistency and non-membership, imply that $\T$ is $(\x_{1:t-1},\hat\y_{1:t-1} )$-prefix-$q$-shattered with respect to $\tilde V_{t-1}$. According to the definition of PPFLdim, we have that 
\begin{equation*}
\begin{aligned}
    {\rm PPFL}_{T-(t-1)}(\x_{1:t-1},\hat{\y}_{1:t-1},\tilde V_{t-1}) \ge q 
    \overset{Eq.~\eqref{eq:qgreater}}{\ge} {\rm PPFL}_{T-(t-1)}(\x_{1:t-1},\hat{\y}_{1:t-1},\tilde V_{t-1}) +1,
\end{aligned}
\end{equation*}
which resembles the desired contradiction (Equation~\eqref{eq:desiredcontradiction}). Therefore, Equation~\eqref{eq:ppfl_le_pfl} holds, which concludes our proof for the upper bound.

\paragraph{Lower bound.}
Let $q:={\rm PFL}_T(\tilde{\H})$. By definition of PFLdim, there exists a depth-$T$ $\Y$-ary tree $\T$ and visible-label
annotations $(f^{\rm vis}_1,\ldots,f^{\rm vis}_T)$ such that $\T$ is $q$-shattered by $\tilde{\H}$.
We describe an adversary strategy that forces any deterministic learner to incur at least $q$ mistakes.

At round $t=1$, the adversary presents $x_1=\T(\varnothing)$. After the learner predicts $\hat y_1$, the adversary reveals
$y^{\rm vis}_1 := f^{\rm vis}_1(\hat y_1)$.
Inductively, after $t-1$ rounds, the adversary has ensured that the instance shown at round $t$ is
\[
x_t=\T(\hat{\y}_{1:t-1}),
\]
i.e., the learner's prediction history indexes the current node. The learner outputs $\hat y_t$, and the adversary reveals
\[
y^{\rm vis}_t := f^{\rm vis}_t(\hat{\y}_{1:t}).
\]
After $T$ rounds, the realized interaction corresponds to a root-to-leaf path
$\hat{\y}_{1:T}\in\Y^T$ in $\T$.
By $q$-shattering, there exists a subset $F^\star\in \tilde{\H}$ such that for all $t\in[T]$,
$y^{\rm vis}_t\in F^\star(x_t)$ (so the interaction is set-realizable), and moreover along this same path there are at least $q$
indices $t$ for which
\[
\hat y_t \notin F^\star(x_t).
\]
Thus the learner makes at least $q$ mistakes, yielding the lower bound
$R^{\star}_{\rm det}(T,\H)\ge q={\rm PFL}_T(\tilde{\H})$.
Combining with the upper bound completes the proof.
\end{proof}

\section{Proof of an Improved Bound for Finite Hypothesis Spaces}

\setcounter{theorem}{20}

\begin{lemma}
\label{lemma:pfldimmono}
Fix $\H$. Then for any $d,d' \in \mathbb{N}$ with $d<d'$,
\begin{equation}
\label{eq:lemma21centralineq}
{\rm PFL}_d(\tilde\H)\le {\rm PFL}_{d'}(\tilde\H).
\end{equation}
\end{lemma}

\begin{proof}
Let $q := {\rm PFL}_d(\tilde\H)$. By the definition of ${\rm PFL}_d(\tilde\H)$, there exist an $\X$-valued, $\Y$-ary full tree $\T$ of depth $d$ and annotation functions
\[
f_1^{\rm vis}, \dots, f_d^{\rm vis},
\qquad
f_t^{\rm vis} : \Y^t \to \Y,
\]
such that $\T$ is $q$-shattered by $\tilde\H$.

We construct an $\X$-valued, $\Y$-ary full tree $\T'$ of depth $d'$ as follows:
\[
\T'(\y) :=
\begin{cases}
\T(\y), & \text{if } |\y| < d,\\[3pt]
\T(\emptyset), & \text{if } d \le |\y| < d'.
\end{cases}
\]
That is, below depth $d$, every newly appended node is labeled by the root instance $\T(\emptyset)$.

Next, define annotation functions
\[
g_1^{\rm vis}, \dots, g_{d'}^{\rm vis},
\qquad
g_t^{\rm vis} : \Y^t \to \Y,
\]
by
\[
g_t^{\rm vis}(\y_{1:t}) :=
\begin{cases}
f_t^{\rm vis}(\y_{1:t}), & \text{if } 1 \le t \le d,\\[3pt]
f_1^{\rm vis}(\y_1), & \text{if } d < t \le d'.
\end{cases}
\]

We claim that $\T'$ is also $q$-shattered by $\tilde\H$. Fix any root-to-leaf path $\y_{1:d'} \in \Y^{d'}$,
and let $\y_{1:d}$ be its prefix of length $d$. Since $\T$ is $q$-shattered by $\tilde\H$, there exists some $F_{\y_{1:d}} \in \tilde\H$ such that:

\begin{equation}
\label{eq:mono_pfla_cond1_prefix}
f_t^{\rm vis}(\y_{1:t}) \in F_{\y_{1:d}}(\T(\y_{1:t-1}))
\quad\text{and}\quad
F_{\y_{1:d}}(\T(\y_{1:t-1})) \in \mathcal{S}(\Y),
\qquad \forall t \in [d],
\end{equation}
and
\begin{equation}
\label{eq:mono_pfla_cond2_prefix}
\bigl| \{\, t \in [d] : y_t \notin F_{\y_{1:d}}(\T(\y_{1:t-1})) \,\} \bigr| \ge q.
\end{equation}

We show that the same $F_{\y_{1:d}}$ witnesses the full path $\y_{1:d'}$ in $\T'$. For $t \le d$, this follows directly from the construction of $\T'$ and $g_t^{\rm vis}$, since
\[
\T'(\y_{1:t-1}) = \T(\y_{1:t-1})
\quad\text{and}\quad
g_t^{\rm vis}(\y_{1:t}) = f_t^{\rm vis}(\y_{1:t}).
\]

Now consider any $t \in \{d+1,\dots,d'\}$. By construction,
\[
\T'(\y_{1:t-1}) = \T(\emptyset)
\quad\text{and}\quad
g_t^{\rm vis}(\y_{1:t}) = f_1^{\rm vis}((y_1)).
\]
On the other hand, by applying \eqref{eq:mono_pfla_cond1_prefix} at $t=1$, we have
\[
f_1^{\rm vis}((y_1)) \in F_{\y_{1:d}}(\T(\emptyset))
\quad\text{and}\quad
F_{\y_{1:d}}(\T(\emptyset)) \in \mathcal{S}(\Y).
\]
Hence,
\[
g_t^{\rm vis}(\y_{1:t}) \in F_{\y_{1:d}}(\T'(\y_{1:t-1}))
\quad\text{and}\quad
F_{\y_{1:d}}(\T'(\y_{1:t-1})) \in \mathcal{S}(\Y),
\qquad \forall t \in \{d+1,\dots,d'\}.
\]

Therefore, condition (1) in the definition of $q$-shattering holds for all $t \in [d']$. Moreover, condition (2) also holds, because the first $d$ rounds already contribute at least $q$ indices by \eqref{eq:mono_pfla_cond2_prefix}. Thus,
\[
\bigl| \{\, t \in [d'] : y_t \notin F_{\y_{1:d}}(\T'(\y_{1:t-1})) \,\} \bigr| \ge q.
\]

Since the choice of $y_{1:d'}$ was arbitrary, every root-to-leaf path of $\T'$ is $q$-witnessed by some element of $\tilde\H$. Hence $\T'$ is $q$-shattered by $\tilde\H$, and therefore
\[
{\rm PFL}_{d'}(\tilde\H) \ge q = {\rm PFL}_d(\tilde\H).
\]
This proves \eqref{eq:lemma21centralineq}.
\end{proof}

\begin{lemma}
\label{lemma:noviseqpred}
Fix $\X,\Y,\S(\Y)$ and $\H$. There exists an (unfull) rooted tree $\T'$ whose longest root-to-leaf path has length at most $T$ with annotation functions $f^{\rm vis}_t$ such that for every root-to-leaf path
$\y_{1:T'}=( y_1,\ldots, y_{T'})$ in $\T'$, there exists a set $F\in \tilde{\H}$ and satisfying
\[
f^{\rm vis}_t(\y_{1:t})\in F\!\left(\T'(\y_{1:t-1})\right)\quad\text{and}\quad
F_{\y_{1:d}}(\T^{\prime}(\y_{1:t-1})) \in \mathcal{S}(\Y),
\qquad  \forall\, t\in[T'],
\]
and along this path there are at least ${\rm PFL}_T(\tilde{\H})$ indices $t$ for which
\[
 y_t \notin F\!\left(\T'(\y_{1:t-1})\right).
\]
Moreover, every edge on $\T'$ satisfies $ y_t\neq f^{\rm vis}_t({\y}_{1:t})$.
\end{lemma}

\begin{proof}
Let $q:={\rm PFL}_T(\tilde{\H})$. By definition, there exists a \emph{full} depth-$T$ tree $\T$ that is $q$-shattered by $\tilde{\H}$.
We construct $\T'$ from $\T$ by repeatedly applying the following pruning operation.

\smallskip\noindent
\emph{Pruning operation.}
Fix any internal node $\T(\y_{1:t-1})$ corresponding to a prefix $\y_{1:t-1}$.
If there exists a child edge labeled by some $y_t$ such that
\begin{equation}
\label{eq:contractible}
y_t=f^{\rm vis}_t(\y_{1:t}),
\end{equation}
then delete the whole subtree starting from the root node $\T(\y_{1:t-1})$, and splice the entire subtree rooted at $\T(\y_{1:t})$ into the edge incident to $\T(\y_{1:t-1})$ (i.e., the edge connecting $\T(\y_{1:t-1})$ to its parent), thereby bypassing $\T(\y_{1:t-1})$. The output of each updated annotation functions $f^{\rm vis}_k$ at every prefix path in this spliced subtree is set to be the one of $f^{\rm vis}_{k+1}$ at the original prefix. Whenever given a node that has multiple child edges satisfying Equation~\eqref{eq:contractible}, we conduct the above operation on arbitrary one to break tie.
Repeat this operation until no such edge remains anywhere in the tree; denote the resulting (possibly unfull) tree by $\T'$.

\smallskip\noindent
We now verify that $\T'$ satisfies the required properties.

\smallskip\noindent
\emph{(i) Edge condition.}
By construction, every remaining edge on $\T'$ satisfies $ y_t\neq f^{\rm vis}_t({\y}_{1:t})$.

\smallskip\noindent
\emph{(ii) The witness conditions $f^{\rm vis}_t(\cdot)\in F(\cdot)$ and $F(\cdot)\in \S(\Y)$ are preserved.}
Consider any root-to-leaf path ${\y}_{1:T'}$ in $\T'$.
This path corresponds to a subsequence of some root-to-leaf path $\y_{1:T}$ in the original full tree $\T$, obtained by deleting the pruned indices where $y_t=f^{\rm vis}_t(\y_{1:t})$.
Let $F\in \tilde{\H}$ be the function set guaranteed by the $q$-shattering of $\T$ for the original path $\y_{1:T}$.
Since the pruning only removes intermediate nodes/levels and does not change the labels assigned to the surviving nodes, the same $F$ continues to satisfy
$f^{\rm vis}_t({\y}_{1:t})\in F(\T'({\y}_{1:t-1}))$ and $F(\T'({\y}_{1:t-1}))\in \S(\Y)$ for all $t\in[T']$.

\smallskip\noindent
\emph{(iii) The number of ``mistake'' indices does not decrease.}
On any pruned index $t$, we had $y_t=f^{\rm vis}_t(\y_{1:t})$.
In particular, along such an edge we cannot have $y_t\notin F(\T(\y_{1:t-1}))$ (since $f^{\rm vis}_t(\y_{1:t})\in F(\T(\y_{1:t-1}))$ on every shattered path).
Hence removing these indices does not remove any step counted in
$\{t:\ y_t\notin F(\T(\y_{1:t-1}))\}$.
Therefore, for every path in $\T'$, the number of indices $t$ with
$ y_t\notin F(\T'({\y}_{1:t-1}))$ is equal to the corresponding number for the original path in $\T$, and thus at least $q$.

\smallskip
Finally, since we only delete levels, the longest root-to-leaf path in $\T'$ has length at most $T$.
This completes the proof.
\end{proof}

\setcounter{theorem}{8}

\begin{theorem}[Improved Bound for the Finite Case]
\label{thm:finitecaseimprovedbound}
    Under partial feedback and set-realizable settings, let $\H \in \Y^{\X}$ denote the hypothesis class with $\abs{\H}=n<\infty$ and $\abs{\Y} = 2$. Then, $\H$ is online learnable under deterministic algorithms with minimax regret no more than $n$.
\end{theorem}

\begin{proof}
Since $\S(\Y)\subseteq \PP(\Y)$, and $\S(\Y)$ determines the scope of valid function sets, the minimax regret is obtained when $\S(\Y)$ is maximal, i.e., $\S(\Y)=\PP(\Y)$. By Theorem~\ref{thm:dldbp}, it suffices to upper bound ${\rm PFL}_T(\tilde{\H})$ for every horizon $T$.
Fix $T$ and let
\[
q := {\rm PFL}_T(\tilde{\H}).
\]
By Lemma~\ref{lemma:noviseqpred}, there exists an (possibly unfull) $\X$-valued, $\Y$-ary tree $\T$ of depth at most $T$
that is $q$-shattered by $\tilde{\H}$, witnessed by visible-label annotation functions
$(f^{\rm vis}_1,\ldots,f^{\rm vis}_{T^{\prime}})$, and moreover along every root-to-leaf path
$\y_{1:T'}=( y_1,\ldots, y_{T'})$ of $\T$ (with $T'\le T$) we have the constraint
\[
 y_t \neq f^{\rm vis}_t(\y_{1:t})\qquad \forall\, t\in[T'].
\]

\noindent Recall that we call a root-to-leaf path $\y_{1:T'}$ $q$-witnessed by a subset $F\in \tilde{\H}$ if
(i) for all $t\in[T']$,
\[
f^{\rm vis}_t(\y_{1:t}) \in F\big(\T(\y_{1:t-1})\big),
\]
and (ii) along this same path there are at least $q$ indices $t\in[T']$ such that
\[
 y_t \notin F\big(\T(\y_{1:t-1})\big).
\]
Since $\T$ is $q$-shattered by $\tilde{\H}$, every root-to-leaf path of $\T$ is $q$-witnessed by some $F\in \tilde{\H}$.

Next we upper bound, for a fixed $F\in \tilde{\H}$, how many depth-$T$ paths it can $q$-witness when $|\Y|=2$.
First, complete $\T$ to a full depth-$T$ binary tree by attaching full binary subtrees below any leaf at depth $T'<T$ as we did in the proof of Lemma~\ref{lemma:pfldimmono}. According to that proof, the completed tree $\T$ is also $q$-shattered by $\tilde \H$, and the non-memberships occur in but not limited to the same positions in the original $\T$. In what follows, we only consider the non-memberships that show up in such positions. Since they limit the number of branches that a function set can witness, counting less such non-memberships exaggerates the number of paths that the function set can witness, so the bound is eligible.

Now fix $F\in \tilde{\H}$ and consider any depth-$T$ path $\y_{1:T}$ that is $q$-witnessed by $F$.
On each of the at least $q$ rounds where $ y_t \notin F(\T(\y_{1:t-1}))$, the next edge choice is \emph{forced}:
because $|\Y|=2$, it holds that $F(\T(\y_{1:t-1})) = \{y_t^{\prime}\}:= \Y\setminus\{ y_t\}$, and $F$ cannot witness the path extended along $\y_{1:t-1} \sqcup ( y_t^{\prime} )$, as it does not obey the constraint
$y^{\prime}_t\neq f^{\rm vis}_t(\y_{1:t-1}\sqcup (y^{\prime}_t))$, where $f^{\rm vis}_t(\y_{1:t-1}\sqcup (y^{\prime}_t))$ must equal to $y^{\prime}_t$, the only element that $F$ outputs at $\T(\y_{1:t-1})$.
Therefore, to keep $F$ consistent with the revealed labels, the path cannot branch arbitrarily at these $q$ rounds.
Only on the remaining at most $T-q$ rounds can the path have up to two choices.
Hence the total number of depth-$T$ paths that a fixed $F$ can $q$-witness is at most
\[
2^{\,T-q}.
\]

Finally, there are $2^n$ function subsets $F\in \tilde{\H}$. Since every depth-$T$ path must be $q$-witnessed by at least one subset, we obtain
\[
2^T \ \le\ \ 2^n\cdot 2^{T-q},
\]
which implies $q\le n$. Since $T$ was arbitrary, this shows ${\rm PFL}_T(\tilde{\H})\le n$ for all $T$, and therefore
the minimax regret is at most $n$.
\end{proof}

\section{Proof of Bounds between Combinatorial Dimensions}

\label{sec:app_rolp}

\begin{algorithm}[t]
\caption{Adversary for Example~\ref{example:relationtomulticlass_app}}
\label{alg:aaer}
\SetKwInOut{Initialize}{Initialize}

\Initialize{
$\tilde Q_0\gets()$, $\hat Q_0\gets()$, $\overline Q_0\gets\emptyset$, $X_0\gets\X$
}

\For{$t=1,2,\ldots,T$}{
  Choose any $x_t\in X_{t-1}$ and show it to the learner;
  
  Learner outputs $\hat{y}_t\in\Y$;

  \For{$j=1,2,\ldots,|\tilde Q_{t-1}|$}{
    Let $(f_1,f_2)\gets(\tilde Q_{t-1})_j$ and $(c_1,c_2)\gets(\hat Q_{t-1})_j$;
    
    \If{$f_1(x_t)=\hat{y}_t$}{$c_1\gets c_1+1$;}
    
    \If{$f_2(x_t)=\hat{y}_t$}{$c_2\gets c_2+1$;}
    
    Update $(\hat Q_{t-1})_j\gets(c_1,c_2)$;
  }

  $S'_t \gets \{f\in\H:\ f(x_t)=\hat{y}_t\}$;
  
  $\overline Q_t \gets \overline Q_{t-1}\cup\{S'_t\}$;

  $W_t \gets (-1,1)\setminus\big(\tilde Q_{t-1}(x_t)\cup \overline Q_t(x_t)\big)$;
  
  Reveal any $y_t^{\rm vis}\in W_t$ to the learner;

  $S_t \gets \{f\in\H:\ f(x_t)=y_t^{\rm vis}\}$;

  Convert $S_t$ into an ordered pair $(f_1,f_2)$;
  
  $\tilde Q_t \gets ( (f_1,f_2),\ \tilde Q_{t-1})$;
  
  $\hat Q_t \gets ( (0,0),\ \hat Q_{t-1})$;

  $X_t \gets X_{t-1}\setminus\{x\in X_{t-1}:\ \exists f\neq f'\ \text{with } f,f'\in\cup_{Q\in\tilde Q_t} Q,\ f(x)=f'(x)\}$\;
}

\For{$j=1,2,\ldots,|\tilde Q_T|$}{
  Let $(f_1,f_2)\gets(\tilde Q_T)_j$ and $(c_1,c_2)\gets(\hat Q_T)_j$;
  
  Add $f_1$ to the ground-truth subset if $c_1\le c_2$, else add $f_2$;
}
\end{algorithm}

In this section, we compare PFLdim with previous combinatorial dimensions. The comparisons are conducted under the same $\X$, $\Y$, and $\H$. The choice of $\S(\Y)$ is detailed in each comparison. Let us start with PFLdim vs. MLdim (Ldim).

\setcounter{theorem}{9}

\begin{theorem}[Partial Feedback vs. Multiclass]
\label{thm:pfm}
    Let $\X$ and $\Y$ be the instance and label spaces, $\H$ be the hypothesis space, $\S_{\rm pf}(\Y)$ and $\S_{\rm mc}(\Y)$ be the subset collections for partial-feedback and multiclass, with $\S_{\rm mc}(\Y)\subseteq \S_{\rm pf}(\Y)$. It holds that ${\rm MLdim}(\H)\le {\rm PFLdim}(\tilde{\H})$. Meanwhile, there exist at least one group of $\X$, $\Y$, and $\H$, such that $\H$ is multiclass, but not partial-feedback online learnable when $\S_{\rm mc}(\Y)\subset \S_{\rm pf}(\Y)$.
\end{theorem}
\begin{proof}
We prove the two statements separately.

\smallskip\noindent
\textbf{(1) $\rm MLdim(\H)\le \rm PFLdim(\tilde{\H})$.}
Fix any depth-$T$ tree $\T$ that is shattered by $\H$ in the multiclass sense (with feedback system $\S_{\rm mc}(\Y)$).
For each root-to-leaf path, let $h\in\H$ be the corresponding multiclass witness hypothesis.
In the partial-feedback shattering definition, assign the ground-truth function subset on that path to be the singleton set $\{h\}$.
This assignment is feasible since $\S_{\rm mc}(\Y)\subseteq \S_{\rm pf}(\Y)$, hence the same visible feedback used in multiclass shattering is also admissible under partial feedback.
Therefore $\T$ is also $T$-shattered by $\tilde{\H}$, implying ${\rm MLdim}(\H)\le {\rm PFLdim}(\tilde{\H})$.
Together with Theorem~\ref{thm:dldbp}, this yields the first half of the theorem.

\smallskip\noindent
\textbf{(2) Strict separation when $\S_{\rm mc}(\Y)\subset \S_{\rm pf}(\Y)$.}
We prove the second half by the following example.

\setcounter{theorem}{20}

\begin{example}\label{example:relationtomulticlass_app}
Let $\X=[-\pi,\pi]$, $\Y=[-1,1]$, $\S_{\rm pf}(\Y)=\sigma(\S_{\rm mc}(\Y))$, and
\[
\H=\{f:\X\to\Y \mid f(x)=\sin(x+c),\ c\in[0,2\pi)\}.
\]
Then $\H$ is multiclass online learnable, but not partial-feedback online learnable in the deterministic set-realizable setting.

\smallskip\noindent
\emph{Multiclass learnability.}
Fix any round and any instance $x\in\X$.
After the learner predicts, the adversary reveals the full label $y\in\Y$.
For this fixed $(x,y)$, the equation $\sin(x+c)=y$ has at most two solutions in $c\in[0,2\pi)$.
Hence there are at most two hypotheses in $\H$ consistent with $(x,y)$, so the version space shrinks to a finite class of size at most $2$ after one labeled example.
A finite multiclass hypothesis class is online learnable with a constant mistake bound, therefore $\H$ is multiclass online learnable.

\smallskip\noindent
\emph{Partial-feedback non-learnability.}
We use the adversary in Algorithm~\ref{alg:aaer}.
It suffices to show two properties:

\smallskip\noindent
\emph{(i) At most one slot increases per round.}
The adversary maintains a multiset of ``slots'' $\hat Q_t$ indexed by the candidate subsets in $\tilde Q_t$ (as in Algorithm~\ref{alg:aaer}).
At round $t$, the instance $x_t$ is chosen from an available set $X_{t-1}$, and then $X_t$ is updated by removing every $x\in X_{t-1}$ for which there exist
$Q,Q'\in \tilde Q_{t}$, and $f\in Q$, $f'\in Q'$ such that $f(x)=f'(x)$.
Consequently, for the prediction value $\hat{y}_t$, there can be at most one subset $Q\in\tilde Q_{t-1}$ that contains some $f$ with
$f(x_t)=\hat{y}_t$; otherwise two distinct subsets would ``collide'' at $x_t$, contradicting the construction of $X_t$.
Therefore, when updating $\hat Q_t$, at most one slot is incremented in each round.

\smallskip\noindent
\emph{(ii) The available instance set never becomes empty.}
For any fixed distinct pair $f(x)=\sin(x+c)$ and $f'(x)=\sin(x+c')$ with $c\neq c'$, the equation $f(x)=f'(x)$ has only finitely many solutions on $[-\pi,\pi]$
(in fact, at most two whenever the two values are not simultaneously $\pm1$).
In Algorithm~\ref{alg:aaer}, at round $t$ we add at most two new hypotheses to the bookkeeping (those satisfying $f(x_t)=y_t^{\rm vis}$), hence the update from
$X_{t-1}$ to $X_t$ removes only finitely many points from $X_{t-1}$.
Since $X_0=\X$ is infinite, it follows inductively that $X_t$ remains infinite (hence nonempty) for all $t$.

\smallskip
Given (i) and the fact that the revealed labels is from $W_t := (-1,1)\setminus\big(\tilde Q_{t-1}(x_t)\cup \overline Q_t(x_t)\big)$, the total number of ``matches'' between the learner's predictions and the hypotheses recorded in the slots equals
\[
r(T):=\sum_{Q\in\hat Q_T} \text{(slot count of $Q$)}\le T.
\]
At the end, the adversary selects in each $Q\in\tilde Q_T$ a ground-truth hypothesis with \emph{minimum} slot count inside that subset (as in Algorithm~\ref{alg:aaer}),
so the learner's disagreement on rounds assigned to that subset is at least half of the total slot count contributed by that subset, hence at least $r(T)/2$ overall.
On the remaining $T-r(T)$ rounds where the learner's prediction matches no recorded subset, the learner is naturally wrong.
Therefore the regret is at least
\[
\frac{r(T)}{2}+(T-r(T)) \;=\; T-\frac{r(T)}{2}\;\ge\;\frac{T}{2},
\]
which is linear in $T$ for every deterministic learner. Hence $\H$ is not partial-feedback online learnable.
\end{example}
This completes the proof of the theorem. 
\end{proof}
The above example is also potentially relative to the applications in signal processing. Under the set-valued setting, it is easy to construct cases which force that the hypothesis class is learnable under partial-feedback, but not under set-valued setting. This is because when instance space $\X$ has only one element, by set-realizable setting, any reveal labels on this instance will be correct in each future round. This guarantees that the learner predicts correct in each of the following round as long as the learner outputs the revealed label shown in the first round. In contrast, this does not put much restriction under the set-valued setting. Just let $\Y$ have infinite elements, $\S(\Y) := \PP(\Y)$, and $\H$ consists of constant function that outputs each of elements in $\Y$. It is not hard to see that such configuration is not learnable under set-valued setting. Conversely, there also exist cases where the learnability holds in partial-feedback, but not in the set-valued setting. Although it seems that there are no obvious bound between these two settings, we manage to identify that when the ratio $\frac{{\rm PFL}_d(\tilde{\H})}{d}$ approaches $1$ as $d\to \infty$, and $\S(\Y)$ satisfies certain constraint, the two learning problems are both unlearnable. We also show that the constraint on $\S(\Y)$ is necessary in the sense that there exists a case where the learnability holds in partial-feedback, but not in the set-valued setting, when this constraint is not satisfied.

\setcounter{theorem}{10}
\begin{theorem}[Partial Feedback vs. Set-Value]
\label{thm:pfsv}
Fix instance/label spaces $\X,\Y$, a hypothesis class $\H$, and assume
$\S_{\rm pf}(\Y)\subseteq \S_{\rm sv}(\Y)$ and that $\S_{\rm sv}(\Y)$ is closed under arbitrary unions
(i.e., for any $\{A_i\}_{i\in I}\subseteq \S_{\rm sv}(\Y)$ we have $\bigcup_{i\in I} A_i\in \S_{\rm sv}(\Y)$).
If
\begin{equation}\label{eq:pfsv_ratio}
\frac{{\rm PFL}_d(\tilde{\H})}{d}\longrightarrow 1
\qquad\text{as }d\to\infty,
\end{equation}
then ${\rm SL}(\H) = \infty$. If $\S_{\rm sv}(\Y)$ is not union-closed, the above implication may fail: there exist $\X,\Y,\H$ and subset collection with $\S_{\rm pf}\subseteq\S_{\rm sv}$ such that Inequality~\eqref{eq:pfsv_ratio} holds, yet ${\rm SL}(\H) < \infty$.
\end{theorem}
\begin{proof}
We prove the two statements separately.

\paragraph{Part I: if ${\rm PFL}_d(\tilde \H)/d \to 1$, then ${\rm SL}(\H) = \infty$.}

We will show that ${\rm PFL}_d(\tilde{\H})$ cannot be too close to $d$; quantitatively,
\begin{equation}\label{eq:key_ineq}
{\rm PFL}_d(\tilde{\H}) \;\le\; d - \Big\lfloor \frac{d}{{\rm SL}(\H) +1}\Big\rfloor.
\end{equation}
Rearranging gives
\[
{\rm SL}(\H) \gtrsim \frac{d}{d - {\rm PFL}_d(\tilde{\H})} 
\]
Assumption \eqref{eq:pfsv_ratio} guarantees that ${\rm SL}(\H)=\infty$. 

It remains to justify Inequality~\eqref{eq:key_ineq}.
By definition of ${\rm PFL}_d(\tilde{\H})$, there exists an $\X$-valued, $\Y$-ary full tree $\T$ of depth $d$
and visible-label annotation functions $(f^{\rm vis}_1,\ldots,f^{\rm vis}_d)$ such that for every root-to-leaf path
$\y_{1:d}=( y_1,\ldots, y_d)$, there exists a witnessing function set
$F_{\y_{1:d}}\in \tilde \H$ with
\[
f^{\rm vis}_t(\y_{1:t}) \in F_{\y_{1:d}}\!\big(\T(\y_{1:t-1})\big)\in \S_{\rm pf}(\Y)
\qquad \forall t\in[d],
\]
and along this path at least ${\rm PFL}_d(\tilde{\H})$ indices satisfy
\[
 y_t \notin F_{\y_{1:d}}\!\big(\T(\y_{1:t-1})\big).
\]
Call such indices \emph{good} on the path, and the remaining indices \emph{bad}.
Thus every root-to-leaf path has at most $d-{\rm PFL}_d(\tilde{\H})$ bad indices.

We first prove the following claim.

\setcounter{theorem}{22}
\begin{claim}
\label{clm:block_bad}
In the top $({\rm SL}(\H)+1)$ levels of $\T$, there exists a root-to-depth-$({\rm SL}(\H)+1)$ path segment
that contains at least one bad index (with respect to the full-path witness $F_{\y_{1:d}}$ extending this segment).
\end{claim}

\begin{proof}
Suppose for contradiction that \emph{every} root-to-depth-$({\rm SL}(\H)+1)$ path segment in the top
$({\rm SL}(\H)+1)$ levels has all its indices good for every leaf extension. 
We will use these PF witnesses to build a valid set-valued shattering configuration of depth ${\rm SL}(\H)+1$.

Fix a node at depth $t-1$ on the top subtree and an outgoing edge labeled by $y_t$
(corresponding to a prefix $\y_{1:t}$). Consider the collection of PF witness sets
\[
W_{t}(\y_{1:t}) \;:=\; \Big\{\,F_{\y_{1:d}}\big(\T(\y_{1:t-1})\big)\ :\ 
\y_{t+1:d}\in \Y^{d-t}\Big\}\ \subseteq\ \S_{\rm pf}(\Y),
\]
ranging over all leaf extensions $\y_{t+1:d}$ of the prefix $\y_{1:t}$.
Define a \emph{prefix-wise} set-valued annotation by taking the union
\begin{equation}
\label{eq:union_annotation}
A_t(\y_{1:t}) \;:=\; \bigcup_{U\in W_t(\y_{1:t})} U.
\end{equation}
Since each $U\in W_t(\y_{1:t})$ lies in $\S_{\rm pf}(\Y)\subseteq \S_{\rm sv}(\Y)$ and $\S_{\rm sv}(\Y)$ is closed under arbitrary unions,
we have $A_t(\y_{1:t})\in \S_{\rm sv}(\Y)$, so these are valid set-valued labels.

Moreover, by the assumption that all indices in the top $({\rm SL}(\H)+1)$ levels are good along every leaf extension,
each $U\in W_t(\y_{1:t})$ excludes the edge label $ y_t$, hence the union $A_t(\y_{1:t})$ also excludes $ y_t$.

Finally, fix any root-to-depth-$({\rm SL}(\H)+1)$ path segment $\y_{1:{\rm SL}(\H)+1}$ and extend it arbitrarily to a leaf
$\y_{1:d}$. Let $F_{\y_{1:d}}\in\tilde\H$ be the corresponding PF witness subset.
For each level $t\le {\rm SL}(\H)+1$, by construction we have
\[
F_{\y_{1:d}}\big(\T(\y_{1:t-1})\big)\in W_t(\y_{1:t})
\quad\Longrightarrow\quad
F_{\y_{1:d}}\big(\T(\y_{1:t-1})\big)\subseteq A_t(\y_{1:t}).
\]
Since $F_{\y_{1:d}}\big(\T(\y_{1:t-1})\big)=\{h(\T(\y_{1:t-1})):h\in F_{\y_{1:d}}\}$, it follows that for any fixed
$h\in F_{\y_{1:d}}$ we have
\[
h\big(\T(\y_{1:t-1})\big)\in A_t(\y_{1:t})\qquad \forall\, t\le {\rm SL}(\H)+1,
\]
which provides the realizability witness required in the definition of set-valued shattering. Therefore the annotations $\{A_t(\cdot)\}_{t\le {\rm SL}(\H)+1}$ shatter $\H$ to depth ${\rm SL}(\H)+1$ under $\S_{\rm sv}(\Y)$,
contradicting the definition of ${\rm SL}(\H)$.
\end{proof}

With Claim~\ref{clm:block_bad} in hand, partition the depth-$d$ tree into consecutive blocks of length ${\rm SL}(\H)+1$.
Starting from the root, apply Claim~\ref{clm:block_bad} to the top block to obtain a depth-$({\rm SL}(\H)+1)$ segment containing a bad index,
and descend to the child subtree rooted at the endpoint of this segment. Repeating this procedure for
$\big\lfloor d/({\rm SL}(\H)+1)\big\rfloor$ blocks yields a single root-to-leaf path in $\T$ that contains at least one bad index in each block,
hence at least $\big\lfloor d/({\rm SL}(\H)+1)\big\rfloor$ bad indices in total.
Since every root-to-leaf path has at most $d-{\rm PFL}_d(\tilde{\H})$ bad indices, we obtain
\[
\Big\lfloor \frac{d}{{\rm SL}(\H)+1}\Big\rfloor \;\le\; d-{\rm PFL}_d(\tilde{\H}),
\]
which is exactly Inequality~\eqref{eq:key_ineq}.
This completes Part I.

\paragraph{Part II: separation---${\rm SL}(\H)<\infty$ when assumption~\eqref{eq:pfsv_ratio} holds but $\S_{\rm sv}(\Y)$ is not union-closed.}

We give an explicit example, and show that when assumption~\eqref{eq:pfsv_ratio} holds but $\S_{\rm sv}(\Y)$ is not union-closed, there exists a deterministic learner such that it incurs no more than $Q<\infty$ loss against any adversary under set-valued and existence-realizable settings. By Theorem~\ref{thm:SLdim}, this implies ${\rm SL}(\H)=Q<\infty$. Meanwhile, there exists an adversary that forces any deterministic learner to fail in each round under partial-feedback and set-realizable setting. This implies assumption~\ref{eq:pfsv_ratio}.
\begin{example}
\label{ex:pf_not_sv}
Fix $\X := \Z$, $\Y := \Z$, $\S_{\rm sv}(\Y) := \S_{\rm pf}(\Y) := \{ 
\{a,b\} \mid a \in \Z_{-} \bigcup \{0\}, \ b\in \Z_{+} \}$, and

\begin{equation*}
    \H :=\Bigg\{ f_{(a,b,c)}=\begin{cases}
b, & x=a,\\
c, & \text{otherwise.}
\end{cases} \quad \Bigg| \quad a\in\Z, \ b\in \Z_{-}\bigcup \{0\}, \ c\in \Z_{+} \Bigg\}.
\end{equation*}
We first show that there exists a deterministic learner such that it incurs no more than $Q<\infty$ loss against any adversary under set-valued and existence-realizable settings.. It suffices to construct a deterministic learner such that it incurs loss no more than $Q=3$ against any adversary. Fix any $t$, we for shorthand let $S_t := \{ a_t,b_t \}$ for some $a_t\in \Z_{-} \bigcup \{0\}$ and $b_t \in \Z_{+}$.

\smallskip

At round $1$, the learner gets $x_1 \in \X$, and predict arbitrary $\hat y_1\in \Y$. Then, the adversary shows label set $S_1$. We know that either $f_{(x_1,a_1,\theta_1)}$ or $f_{(\theta_1,\theta_1,b_1)}$ is the witness function. 

\smallskip

At round $2$, the adversary shows $x_2$. The learner predicts $\hat y_2 = b_1$, and the adversary shows $S_2$. 

\smallskip

\quad If $b_2 \neq b_1$, then $a_2$ must equal $a_1$, and $f_{(x_1,a_1,\theta_1)}$ is the witness function. 

\smallskip

\quad Otherwise, if $x_2 = x_1$, then the learner is correct and repeats the strategy for round $2$ in the next round. 

\smallskip

\quad If $x_2 \neq x_1$, then $a_2$ must equal $a_1$, and $f_{(x_1,a_1,b_2)}$ is the witness function. 

\smallskip

At the first round $t$ right after the learner fails since round $2$, we know that either $f_{(x_1,a_1,\theta_1)}$ or $f_{(x_1,a_1,b_2)}$ is the witness function. 

\smallskip

\quad If we know that $f_{(x_1,a_1,b_2)}$ is the witness function, then the learner predicts $b_2$ when $x_{t} \neq x_1$, and $a_1$ when $x_{t} = x_1$. By this the learner will never fail anymore.

\smallskip

\quad If we know that $f_{(x_1,a_1,\theta_1)}$ is the witness function, then the learner predicts arbitrary $\hat y_{t}$ when $x_{t} \neq x_1$, and $a_1$ when $x_{t} = x_1$. The adversary shows $S_t$. If $x_{t} = x_1$, the learner must be correct. If $x_{t} \neq x_1$, the learner could fail but will know that $f_{(x_1,a_1,b_t)}$ is the witness function and will not fail in the future rounds anymore. 

\smallskip

As we have shown, the learner will not incur loss more than $3$ against any adversary, implying that $\H$ is online learnable under set-valued and existence-realizable settings.

Next, we show that under partial-feedback and set-realizable settings, there exists such an adversary that forces any deterministic learner to fail in all rounds. At round $t$, the adversary's strategy consists of: (1) choosing $x_t$ that has not appeared in the previous rounds and (2) revealing $y^{\rm vis}_t \neq \hat y_t$, with $y^{\rm vis}_t \in \Z_{-} \bigcup \{ 0 \}$. After all $T$ rounds, we have that $f_{(x_1,y^{\rm vis}_1,\theta_1)},\ldots,f_{(x_T,y^{\rm vis}_T,\theta_T)}$ must \textbf{all} be witness functions with fixed $(\theta_1,\ldots,\theta_T)\in \Z_{+}$. Let $\theta_1 = \dots = \theta_T = u$ for any $u\neq \hat y_t$ for all $t\in [T]$. We also let $f(x,y_x,u)$ be a witness function for each $x\in \X$ and any fixed associated $y_x\in \Z_{-} \bigcup \{0\}$. By this construction, the witness label set for $x$ is $\{ y_x, u \}\in \S_{\rm pf}(\Y)$ if $x\notin \{ x_1,\ldots, x_T \}$, and $\{ y^{\rm vis}_t, u \} \in \S_{\rm pf}(\Y)$ if $x=x_t$. Thus, it satisfies the set-realizable condition, and the learner fails in all $T$ rounds. Theorem~\ref{thm:dldbp} further implies that assumption~\eqref{eq:pfsv_ratio} holds.

So far we have shown that when assumption~\eqref{eq:pfsv_ratio} holds but $\S_{\rm sv}(\Y)$ is not union-closed, there exists a deterministic learner such that it incurs no more than $Q=3<\infty$ loss against any adversary under set-valued and existence-realizable settings, there exists an adversary that forces any deterministic learner to fail in each round under partial-feedback and set-realizable setting. This implies that if assumption~\eqref{eq:pfsv_ratio} holds and $\S_{\rm sv}(\Y)$ is not union-closed, then there exist $\X,\Y,\H$ and subset collection with $\S_{\rm pf}\subseteq\S_{\rm sv}$ such that Inequality~\eqref{eq:pfsv_ratio} holds, yet ${\rm SL}(\H) < \infty$.
\end{example}

Combining both parts completes the proof.
\end{proof}
\begin{remark*}
The implication from partial feedback to set-valued unlearnability requires additional structure, the union-closure property, on
$\S_{\rm sv}(\Y)$ beyond $\S_{\rm pf}(\Y)\subseteq \S_{\rm sv}(\Y)$.
Indeed, when $\S_{\rm sv}(\Y)$ is not union-closed, there exist instances (Example~\ref{ex:pf_not_sv})
for which the set-valued problem is learnable with constant minimax regret, while the partial-feedback problem is not learnable.
Consequently, without union-closure assumption (used to justify the choice of witness label sets in the proof), the conclusion of Theorem~\ref{thm:pfsv} does not hold in general.
\end{remark*}

It remains an open problem whether $\frac{N-{\rm PFL}(\tilde{\H})}{N}\overset{N\to \infty}{\longrightarrow} 1$ as the condition can be further tightened to ensure that the set-valued and partial-feedback online learning are both not learnable.

\section{Proof of Learnability with Randomized Learners}

\begin{algorithm}[t]
\caption{Fixed-scale Randomized Partial-Feedback Learner (FRPFL)}
\label{alg:frpfl}
\SetKwInOut{Initialize}{Initialize}

\Initialize{$\tilde{V}_0=\tilde\H$, $\x_{1:0},\widehat{\Pi}_{1:0}=()$, $\gamma\in [0,1]$;}

\For{$t=1,\dots,T$}{

  receive $x_t \in \mathcal{X}$;
  
  update $\x_{1:t} \gets \x_{1:t-1} \sqcup (x_t)$;

  $Q_t \gets \bigcup\limits_{F\in \tilde{V}_{t-1}} F(x_t) $;
  
  predict $\hat{\pi}_{t} \gets \argmin\limits_{\hat{\pi}\in\Pi(\Y)} \max\limits_{y\in Q_t} {\rm PPMS}_{(T-t,\gamma)} \left(\x_t, \hat{\Pi}_{1:t-1} \sqcup (\hat{\pi}), \tilde{V}^{(x_t,y)}_{t-1}\right) $,

  where $\tilde{V}^{(x_t,y)}_{t-1} \gets \left\{S\in \tilde{V}_{t-1}: \exists\ h\in S,\ {\rm s.t.} \ h(x_t)=y \right\}$;

  $\widehat{\Pi}_{1:t} \gets \widehat{\Pi}_{1:t-1} \sqcup \hat{\pi}_t$;

  receive $y^{\rm vis}_t\in Q_t$;

  $\tilde{V}_t \gets \tilde{V}^{(x_t,y^{\rm vis}_t)}_{t-1}$;
}
\end{algorithm}

\begin{algorithm}[t]
\caption{Multi-scale Randomized Partial-Feedback Learner (MRPFL)}
\label{alg:mrpfl}
\SetKwInOut{Initialize}{Initialize}

\Initialize{$\tilde{V}_0=\tilde \H$, $\x_{1:0},\widehat{\Pi}_{1:0}=()$, $\gamma_i = \frac{1}{2^i}$ for $i\in [N]$;}

\For{$t=1,\dots,T$}{

  receive $x_t \in \mathcal{X}$;
  
  update $\x_{1:t} \gets \x_{1:t-1} \sqcup (x_t)$;

  $Q_t \gets \bigcup\limits_{F\in \tilde{V}_{t-1}} F(x_t) $;

  \For{$i\in [N]$}{
    let $\pi^{i}_t \gets \argmin\limits_{\hat{\pi}\in\Pi(\Y)} \max\limits_{y\in Q_t} {\rm PPMS}_{(T-t,\gamma)}\left(\x_t,\widehat{\Pi}_{1:t-1}\sqcup (\hat{\pi}),\tilde{V}^{(x_t,y)}_{t-1}\right) $, \\ where $\tilde{V}^{(x_t,y)}_{t-1} \gets \left\{S\in \tilde{V}_{t-1}: \exists\ h\in S,\ {\rm s.t.} \ h(x_t)=y \right\}$;
  }
  predict $\hat{\pi}_t \gets \pi^{m_t}_t$, where $m_t := {\rm MSP}( N, \{\pi^i_{t}\}^N_{i=1}, \{\gamma_i\}^{N}_{i=1}, \S(\Y) )$;
  
  $\widehat{\Pi}_{1:t} = \widehat{\Pi}_{1:t-1}\sqcup \hat{\pi}_t$;

  receive $y^{\rm vis}_t\in Q_t$;

  $\tilde{V}_t \gets \tilde{V}^{(x_t,y^{\rm vis}_t)}_{t-1}$;
}
\end{algorithm}

\begin{algorithm}[t]
\caption{Measure Selection Procedure~(MSP, Algorithm 3 in \citep{onlinelearningsetvaluefeedback}}
\label{alg:msp}
\SetKwInOut{Input}{Input}

\Input{Total number of grid points $N$, sequence of measures $\pi_1,\dots,\pi_N\in \Pi(\Y)$, sequence of thresholds $\gamma_1,\dots,\gamma_N \in (0,1)$, subset collection $\S(\Y)\in \sigma(\Y)$.}

\For{$m\in [N-1]$}{
\If{
\begin{equation*}
    \sup_{S\in \S(\Y)} \abs{ \pi_i(S^{\cc}) - \pi_{i-1}(S^{\cc}) } \le 2\gamma_{i-1}
\end{equation*}
holds for all $2\le i \le m$, but we also have
\begin{equation*}
    \inf_{S\in \S(\Y)} \abs{ \pi_m(S^{\cc}) - \pi_{m+1}(S^{\cc}) } \ge 2\gamma_m,
\end{equation*}}
{
return $m$ and quit;}
}
return $N$;
\end{algorithm}

\begin{algorithm}
    
\end{algorithm}

\begin{definition}[Prefix Partial-Feedback Measure Shattering Dimension (PPMSdim)]
\label{def:ppms_inline}
Fix $n\ge 0$ and $\gamma\in[0,1]$. Given a prefix $\x_{1:n}\in\X^n$ and $\Pi_{1:n}=(\pi_1,\ldots,\pi_n)\in\Pi(\Y)^n$, define ${\rm PPMS}_{(d,\gamma)}(\x_{1:n},\Pi_{1:n},\tilde{V})$ as the largest $q$ such that there exists an $\X$-valued depth-$d$ tree $\T:\Pi(\Y)^{<d}\to\X$ and visible-label annotations $f^{\rm vis}_t:\Pi(\Y)^{t-n}\to\Y$ for $t\in\{n+1,\ldots,n+d\}$ with the property that for every continuation $\Pi_{n+1:n+d}\in\Pi(\Y)^d$ there is a subset $F_{\Pi_{n+1:n+d}}\in \tilde{V}$ satisfying: (i) for all $s\in[d]$, $f^{\rm vis}_{n+s}(\Pi_{1:n+s})\in F_{\Pi_{n+1:n+d}}(\T(\Pi_{n+1:n+s-1}))\in\S(\Y)$; (ii) along the combined history (prefix + continuation) at least $q$ indices satisfy $\pi_t(F_{\Pi_{n+1:n+d}}(x_t))\le 1-\gamma$ (for some $t\in[n]$) or $\pi_{n+s}(F_{\Pi_{n+1:n+d}}(\T(\Pi_{n+1:n+s-1})))\le 1-\gamma$ (for some $s\in[d]$), with $\le 1-\gamma$ replaced by $<1$ when $\gamma=0$.
\end{definition}

\begin{lemma}
\label{lemma:mrpflbound}
    For MRPFL, fix $t\in [T]$, if $\hat{\pi}_t (S^{\cc}_t) \ge \gamma_N$, then there exists a $j\in [N]$, such that $\gamma_j \ge \frac{\hat{\pi}_t(S^{\cc}_t)}{16}$ and $\pi^j_t(S_t^{\cc})\ge \gamma_j$ both hold.
\end{lemma}
\begin{proof}
    Same as the proof of Lemma~13 in~\cite{onlinelearningsetvaluefeedback}.
\end{proof}

\begin{lemma}
\label{lemma:mostbeyondgamma}
    Fix $\X$, $\Y$, $\S(\Y)$, and $\H$. Let $( (x_1, S_1, \hat{\pi}_{1}), (x_2, S_2, \hat{\pi}_{2}), \dots, (x_T, S_T, \hat{\pi}_{T}) )$ be the $T$-round game played by FRPFL of $\gamma$ against the any adversary. Then, there are at most ${\rm PMS}_{(T,\gamma)}(\tilde{\H})$ rounds $(x_t,S_t,\hat{\pi}_t)$ where $\hat{\pi}_t(S^{\cc}_t) \ge \gamma$.
\end{lemma}

\begin{proof}
    The proof is similar to that for the upper bound in Theorem~\ref{thm:dldbp}. Replace ${\rm PFL}_T(\tilde\H)$ and ${\rm PPFL}_{T-t}(\x_{1:t},\hat{\y}_{1:t},\tilde{V}_t)$ with ${\rm PMS}_{(T,\gamma)}(\tilde{\H})$ and ${\rm PPMS}_{(T-t,\gamma)}(\x_{1:t},\hat{\Pi}_{1:t},\tilde{V}_t)$, respectively; replace $\hat y_t \notin S_t$ with $\hat\pi(S_t)\le 1 - \gamma$, and the proof is generally identical.
\end{proof}

\setcounter{theorem}{12}

\begin{theorem}[Regret of Randomized Algorithms]
\label{thm:regretofrand}
    For any $\X$, $\Y$, and $\H$, for all randomized algorithms $\A$, the minimax regret of any $T$-round game satisfies that
    \begin{equation}
    \tag{4}
    \label{eq:4}
        \sup_{\gamma\in (0,1]} \gamma {\rm PMS}_{(T,\gamma)}(\tilde{\H}) \le \inf_{\A} R_{\A}(T,\H) \le  C \inf_{\gamma\in (0,1]} \left\{ \gamma T + \int^{1}_{\gamma} {\rm PMS}_{(T,\eta)}(\tilde{\H}) \dd\eta \right\},
    \end{equation}
    where $C$ is some universal positive constant. The upper and lower bounds are tight up to constant factors.
\end{theorem}

\begin{proof}
    The proof is mainly motivated by~\cite{onlinelearningsetvaluefeedback} and \cite{daskalakis2022fast}. The difference is that~\cite{onlinelearningsetvaluefeedback}'s argument relies on the descending property of SLdim. In contrast, our PFLdim and PPFLdim do not share this property but only maintain an upper bound. To tackle this, we change the associated part of proof to a counting argument.

    We first prove the upper bound direction. Suppose that for any adversary, the game has total of $T$ rounds, and the associated sequence is $(x_1, S_1, \hat{\pi}_1), (x_2, S_2, \hat{\pi}_2),\dots, (x_T, S_T, \hat{\pi}_T)$. The first goal is to show an upper bound of Algorithm~\ref{alg:mrpfl} (MRPFL), as 
    \begin{equation*}
        \sum^T_{t=1} \hat{\pi}_t(S^{\cc}_t) \le \gamma_N T + 16\sum^N_{i=1} \gamma_i {\rm PMS}_{(T,\gamma_i)} (\tilde{\H}).
    \end{equation*}
    Since this inequality holds for any adversary, we will have that
    \begin{equation*}
        \inf_{\A} R_{\A}(T,\H) \le R_{\rm MRPFL}(T,\H) \le \gamma_N T + 16\sum^N_{i=1} \gamma_i {\rm PMS}_{(T,\gamma_i)}(\tilde\H).
    \end{equation*}
    Instead of the telescope arguments in~\cite{onlinelearningsetvaluefeedback}, it suffices to show that

    \begin{equation}
    \label{eq:3}
        \gamma_N + 16 \sum^N_{i=1} \1\left\{ \gamma_i \le \pi^i_t(S^{\cc}_t) \right\}  \gamma_i \ge \hat{\pi}_t (S^{\cc}_t). 
    \end{equation}
    This is because in MRPFL, each of $\pi^i_t$ is generated in the minimax style as in FRPFL, according to Lemma~\ref{lemma:mostbeyondgamma}, we have that
    \begin{equation*}
        \sum^T_{t=1}\1\left\{ \gamma_i \le \pi^i_t(S^{\cc}_t) \right\} \le {\rm PMS}_{(T,\gamma_i)}(\tilde{\H}).
    \end{equation*}
    Plug it into Inequality~(\ref{eq:3}), and we get
    \begin{equation*}
        \begin{aligned}
            \sum^T_{t=1} \hat{\pi}_t (S^{\cc}_t) &\le \sum^T_{t=1}\gamma_N + 16 \sum^N_{i=1} \sum^T_{t=1} \1\left\{ \gamma_i \le \pi^i_t(S^{\cc}_t) \right\}  \gamma_i \\
            &\le T\gamma_N + 16 \sum^N_{i=1} \gamma_i {\rm PMS}_{(T,\gamma_i)} (\tilde{\H}).
        \end{aligned}
    \end{equation*}
    Now, we proceed to prove Inequality~(\ref{eq:3}). First, notice that if $\hat{\pi}_t (S^{\cc}_t) \le \gamma_N$, then Inequality~(\ref{eq:3}) naturally holds. Thus, we discuss the condition when $\hat{\pi}_t (S^{\cc}_t) > \gamma_N$. It suffices to prove that there exists a $j\in [N]$, such that $\gamma_j \ge \frac{\hat{\pi}_t(S^{\cc}_t)}{16}$ and $\pi^j_t(S_t^{\cc})\ge \gamma_j$ both hold for every $t\in [T]$. This is because it can give us 
    \begin{equation*}
    \begin{aligned}
        \gamma_N + 16 \sum^N_{i=1} \1\left\{ \gamma_i \le \pi^i_t(S^{\cc}_t) \right\}  \gamma_i  &\ge 16 \1\left\{ \gamma_j \le \pi^j_t(S^{\cc}_t) \right\}  \gamma_j
        \\ &\ge  16  \cdot 1\cdot \frac{\hat{\pi}_t(S^{\cc}_t)}{16} \\
        &\ge \hat{\pi}_t (S^{\cc}_t). 
    \end{aligned}
    \end{equation*}
    via plugging these two inequalities into Inequality~(\ref{eq:3}). The inequalities are guaranteed by Lemma~\ref{lemma:mrpflbound}. Then it only remains to prove
    \begin{equation*}
        \gamma_{N} T + 16 \sum^N_{i=1} \gamma_i {\rm PMS_{(T,\gamma_i)}(\tilde{\H})} \le C \inf_{\gamma>0} \left\{ \gamma T + \int^1_{\gamma} {\rm PMS}_{(T\eta)}( \tilde{\H} ) \dd \eta \right\},
    \end{equation*}
    for some constant $C>0$. The proof for this is identical to that for Lemma~13 in \cite{onlinelearningsetvaluefeedback}, by leveraging the property of MSP. So far, we complete the proof for the upper bound.

    Next, we prove the lower bound. It suffices to show that for each of the $\gamma \in (0,1]$, there exists an adversary that can lead to a regret of $\gamma {\rm PMS}_{(T,\gamma)} ( \tilde{\H})$. If this is proved, the regret is higher than the maximum regret amongst these adversaries, i.e., 
    \begin{equation}
    \label{eq:5}
        \sup_{\gamma\in (0,1]} \gamma {\rm PMS}_{(T,\gamma)}(\tilde{\H}) \le \inf_{\A} R_{\A}(T,\H).
    \end{equation}
    Now, we begin to prove the lower bound of each adversary. Given any $\gamma\in (0,1]$, by the definition of ${\rm PMSdim}$ there exists a tree $\T$ of depth $T$ that is $\gamma$-measure ${\rm PMS}_{(T,\gamma)}(\tilde{\H})$-shattered by $\tilde{\H}$. The strategy of adversary is similar to the proof of the lower bound in Theorem~\ref{thm:dldbp}: The adversary always reveals the label output by $f^{\rm vis}_t( \hat\Pi_{1:t-1} \sqcup \hat{\pi}_{t})$ for the $\hat{\pi}_{t}$ that the learner chooses in each round. This will guarantee in at least ${\rm PMS}_{(T,\gamma)}(\tilde{\H})$ rounds, the error expectation is higher than or equal to $\gamma$. Put them together, and we will get the lower bound of the minimax regret, i.e., Inequality~\eqref{eq:5}.

    Finally, we show the tightness. For the upper bound, we show that there exists a matched lower bound for some $\X$, $\Y$, $\S(\Y)$, and $\H$. Let $\Y$ have only two elements, $\S(\Y)$ be a set of singletons, suppose the game will run for $T$ rounds, and let ${\rm PFL}_T(\tilde{\H})$ be the PFLdim of $T$-depth. This guarantees that there exists a depth-$T$ binary tree $\T$ that is ${\rm PFL}_T(\tilde{\H})$-shattered by $\tilde{\H}$. The adversary chooses $x_0=\T(\emptyset)$ to kick off the game. In the following each round $t$, the adversary reveals the label $f^{\rm vis}_t(\hat{\y}_{1:t-1}\sqcup (y_t))$ of the edge indexed by the $y_t\in \Y$ that the output measure put on higher possibility. It is not hard to see that the best strategy for the learner is to always predict each label with the half possibility. This results in a tight lower bound of $\frac{{\rm PFL}_T(\tilde{\H})}{2}$. Notice that the Helly number of this scenario is finite. Thus, by Theorem~\ref{thm:sufficientnotnecessary}, we have that ${\rm PMS}_{T,0}(\tilde{\H}) = {\rm PFL}_T(\tilde{\H})$. Plug into the right hand side of Inequality~(\ref{eq:4}), and we have that $\inf_{\A} R_{\A}(T,\H) \le C \cdot {\rm PFL}_{T}(\tilde{\H})$. Our upper bound matches the lower bound $\frac{{\rm PFL}_T(\tilde{\H})}{2}$ up to a constant.

    For the lower bound, in contrast to the upper bound, we need to show that there exists a matched upper bound for some $\X$, $\Y$, $\S(\Y)$, and $\H$. This is a modified version of the one used by~\cite{onlinelearningsetvaluefeedback}, as the original version used by them will break down in the partial-feedback setting due to the limited hypothesis class. Let $\X = \N_{+}$, $\Y = \{1,2,3,4,5,6\}$, $\S(\Y) = \{ \{ 1,2,4 \},\{3,4,6 \}, \{2,5,6\} \}$, and $\H = \{ f_i \equiv i \mid i\in [6] \}$. To prove an upper bound, it suffices to show a learner strategy against which the fiercest adversary can only force the learner to incur loss with expectation $\frac{1}{3}$. This is because the Helly number is $3$, which gives us ${\rm PFL}_{T}(\tilde{\H}) = {\rm PMS}_{(T,\frac{1}{3})} ( \tilde{\H} )$. Obviously ${\rm PFL}_{T}(\tilde{\H}) = 1$, so ${\rm PMS}_{(T,\frac{1}{3})} ( \tilde{\H} ) = 1$. Plug it into the left hand side of Inequality~\eqref{eq:4}, and we get $\inf_{\A} R_{\A} (T,\H) \ge \sup_{\gamma \in (0,1]} \gamma {\rm PMS}_{(T,\gamma)} (\tilde{\H}) \ge \frac{1}{3}$. Thus, it suffices to prove an $\frac{1}{3}$ upper bound. Let the learner predict $(2,4,6)$ with the same $\frac{1}{3}$ probability in the first round. After the adversary reveals label $y^{\rm vis}_1$ in the first round, let the learner predict $y^{\rm vis}_1$ in all the following rounds. It is not hard to see that by this strategy, the learner suffers regret $\frac{1}{3}$ as a matching upper bound. 
\end{proof}

\section{Proof of Learnability Consistency}

\setcounter{theorem}{15}

\begin{theorem}
\label{thm:sufficientnotnecessary}
    Let $\H$ be a hypothesis class, and $\S(\Y)\subseteq \sigma(\Y)$.  If there exists a finite positive integer $p$, and for every subcollection $\mathcal{C}\in\S(\Y)$ with $\bigcap_{S \in \mathcal{C}} S = \emptyset$, either one of the following two conditions is satisfied: (1) there exists a countable sequence $\{S_i\}^{\infty}_{i=1} \subseteq \mathcal{C}$ with $S_1\supseteq S_2\supseteq \dots$, satisfying $\bigcap^{\infty}_{i=1} S_i = \emptyset$; (2) there exists a subcollection $\S \subseteq \mathcal{C}$ of size $\abs{\S} = p$, satisfying $\bigcap_{S\in \S} S = \emptyset$, then ${\rm PMS}_{(d,\gamma)}(\tilde{\H}) = {\rm PFL}_d(\tilde{\H})$ holds for all $\gamma\in [0,\frac{1}{p}]$. Additionally, if all the $\mathcal{C}$ with $\bigcap_{S \in \mathcal{C}} S = \emptyset$ satisfy (1), then ${\rm PMS}_{(d,\gamma)}(\tilde{\H}) = {\rm PFL}_d(\tilde{\H})$ holds for all $\gamma\in [0,1)$. Meanwhile, there exists a $\S(\Y)$ of which the Helly number is infinity, but the above condition is satisfied.
\end{theorem}

\begin{proof}
 We need to show that ${\rm PMS}_{(d,\gamma)}(\tilde{\H}) \ge {\rm PFL}_{d}(\tilde{\H})$ and ${\rm PMS}_{(d,\gamma)}(\tilde{\H}) \le {\rm PFL}_{d}(\tilde{\H})$ both hold under the conditions stated above.

 We start with ${\rm PMS}_{(d,\gamma)}(\tilde{\H}) \ge {\rm PFL}_d(\tilde{\H})$. Without loss of generality, let $\T$ be a tree of depth $d$ that is ${\rm PFL}_{d}(\tilde{\H})$-shattered by $\tilde{\H}$. We show that, under the conditions in this theorem's statement, there exists a tree $\T^{\prime}$ of depth $d$ that is $\gamma$-measure ${\rm PFL}_{d}(\tilde{\H})$-shattered by $\tilde{\H}$, where $\gamma \in [0,\frac{1}{p}]$ or $\gamma \in [0,1)$, depending on the concrete conditions satisfied. The construction is node-to-node, in the sense that we construct the node $\T^{\prime}({\Pi}_{1:i})$ based on $\T(\y_{1:i})$ between some pair of $({\Pi}_{1:i},\y_{1:i})$ with $\T^{\prime}({\Pi}_{1:i})=\T(\y_{1:i})$. We start from $\T^{\prime}({\Pi}_{1:0}) = \T(\y_{1:0})$ with ${\Pi}_{1:0}=\y_{1:0}=\emptyset$, and update the pair based on some rules to continue constructing till $\T^{\prime}$ is completed. Now, we introduce the mechanism. Suppose we are currently constructing $\T^{\prime}({\Pi}_{1:i})$ based on $\T(\y_{1:i})$. First notice that, if there exists a $y\in \Y$, we have that $y\in S$ holds for all $S \in \S(\Y)$, then ${\rm PFL}_d(\tilde{\H}) = 0$ and ${\rm PMS}_{(d,\gamma)}(\tilde{\H}) \ge {\rm PFL}_{d}(\tilde{\H})$ holds naturally. We further continue the proof based on the condition that for any of the $y\in \Y$, there exists $S$ such that $y\notin S \in \S(\Y)$. We assert that if either of the conditions (1) and (2) holds, then for every measure $\pi\in \Pi(\Y)$ and $0< \gamma <1$, there exists a $y\in \Y$, such that for all $S\in \S(\Y)$ with $y\notin S$, we have either $\pi(S) < 1-\gamma$, or $\pi(S) \le 1-\frac{1}{p}$. Let $\mathbf{S}_y$ denote the collection of all sets $S_y\subseteq \S(\Y)$ with $y\notin S_y$. Based on the above assumption, we have that $\abs{\mathbf{S}_y}\ge 1$ holds for all $y\in \Y$. We select a $S_{y}$ for each $y\in \Y$ to form a set, $\mathcal{S}:=\{S_{y}\}_{y\in \Y}$. Since apparently $\bigcap_{y\in \Y} S_{y} = \emptyset$, $\S$ as the $\mathcal{C}$ in the statement, satisfies either (1) or (2). When $\mathcal{S}$ satisfies (1), we prove that for every measure $\pi \in \Pi(\Y)$ and $0 < \gamma < 1$, there exists a set $S\in\mathcal{S}$, satisfying that $\pi(S)< 1 - \gamma$. We prove by contradiction, assuming that there exists a measure $\pi\in \Pi(\Y)$ and $\gamma \in (0,1)$, such that for every set $S\in \S$, we have that $\pi(S_{y})\ge 1-\gamma$. According to the definition of (1), there exists a subcollection $\widehat{\S}:=\{S_{y_{1}},S_{y_{2}}, \dots \} \subseteq \S$, of which $S_{y_{1}} \supseteq S_{y_{2}} \supseteq \dots$ and $\bigcap^{\infty}_{i=1} S_{y_{i}}$ hold. Since $S_{y_{1}} \supseteq S_{y_{2}} \supseteq \dots$, according to the continuity from above, we have that 
 \begin{equation*}
 \lim_{i\to \infty}\pi(S_{y_{i}}) = \pi(\bigcap^{\infty}_{i=1} S_{y_{i}} ) = \pi(\emptyset) = 0. 
 \end{equation*}
 However, since $\pi(S_{y_{i}})\ge 1-\gamma$ holds for all $S_{y_{i}}\in \widehat{\S}$, we have that $\lim_{i\to \infty}\pi(S_{y_{i}}) \ge 1-\gamma>0$, which gives us a contradiction. 
 
 Next, we show that if it satisfies (2), then for every measure $\pi\in \Pi(\Y)$, there must exist a set $S\in \S$, satisfying that $\pi(S) \le 1 - \frac{1}{p}$. We prove by contradiction, assuming that there exists a measure $\pi$, for which $\pi(S) > 1 - \frac{1}{p}$ holds for all $S\in {\S}$. Condition (2) implies that, there exists a subcollection $\widehat{\S} \subseteq \mathcal{S}$ of size $\abs{\widehat{\S}} \le p$, such that $\bigcap_{S\in\widehat{\S}} S =\emptyset$.  Since $\bigcap_{S \in \widehat{\S}} S = \emptyset$, we have that $\bigcup_{S\in \hat{\S}} S^{\cc} = \Y$. This gives us $\pi(\Y)=\pi(\bigcup_{S\in \hat{\S}} S^{\cc}) \le \sum_{S\in \hat{\S}} \pi(S^{\cc}) < \abs{\widehat{\S}} \frac{1}{p} \le 1$, which contradicts the fact $\pi(\Y)=1$. Thus, for every measure $\pi\in \Pi(\Y)$, there must exist a set $S\in {\S}$, satisfying that $\pi(S) \le 1 - \frac{1}{p}$.

The above two proofs under conditions (1) and (2) can be extended forward to prove the original statement, i.e., if either conditions (1) or (2) holds, for every measure $\pi\in \Pi(\Y)$ and $0<\gamma<1$, there exists a $y\in \Y$, such that for all $S\in \S(\Y)$ with $y\notin S$, we have either $\pi(S) < 1-\gamma$, or $\pi(S) \le 1-\frac{1}{p}$. We suppose, for contradiction, this is false, namely, there exists a measure $\pi\in \Pi(\Y)$ and $0<\gamma<1$, for all $y\in \Y$, there exists a $S\in \S(\Y)$ with $y\notin S$, such that we have $\pi(S) > 1-\frac{1}{p}$ and $\pi(S) \ge 1-\gamma$. Then, we can choose such a $S$ for each $y$ to form a collection $\S$. However, according to the above proof, for every measure $\pi\in \Pi(\Y)$ and $0<\gamma<1$, there exists a $S\in \S$, such that either $\pi(S)<1-\gamma$ or $\pi(S)\le 1-\frac{1}{p}$ holds, which contradicts the above construction, and we reach a contradiction. Therefore, for each of the edge indexed by $\pi\in\Pi(\Y)$ outgoing from $\T^{\prime}({\Pi}_{1:i})$, there must exist a $y$ for which all $S_y$ with $y\notin S_y$ satisfy $\pi(S_y) < 1 - \gamma$ or $\pi(S_y) \le 1 - \frac{1}{p}$. Thus, all the $S_y$ output by the ground-truth function set $F_{\y_{1:d}}$ of each path $\y_{1:d}$ passing $\y_{1:i} \sqcup (y)$ at $\T(\y_{1:i})$ must satisfy $\pi(S_y) < 1 - \gamma$ or $\pi(S_y) \le 1 - \frac{1}{p}$ as well. We annotate each edge $\pi$ with the associated revealed label $y^{\rm vis}$, and pair ${\Pi}_{1:i} \sqcup (\pi)$ with $\y_{1:i}\sqcup (y)$. It is not hard to see that after iteratively operations, $\T^{\prime}$ of depth $d$ can be constructed that is $\frac{1}{p}$-measure ${\rm PFL}_d(\H)$-shattered by $\tilde{\H}$, and if all $\pi(S_y) < 1 - \gamma$ holds, then $\T^{\prime}$ is $\gamma$-measure ${\rm PFL}_d(\H)$-shattered by $\tilde{\H}$ for all $\gamma\in [0,1)$.

This completes the proof for ${\rm PMS}_{(d,\gamma)}(\tilde{\H}) \ge {\rm PFL}_d(\tilde{\H})$. Next, we show that ${\rm PMS}_{(d,\gamma)}(\tilde{\H}) \le {\rm PFL}_{d}(\tilde{\H})$.
 
By the definition of PMSdim, there exists a tree $\T$ of depth $d$ that is $\gamma$-measure ${\rm PMS}_{(d,\gamma)}(\tilde{\H})$-shattered by $\tilde{\H}$. We prune the edges other than the delta measures $\delta_y\in \Pi(\S(\Y))$ in each layer for all $y\in \Y$. The depth of $\T$ remains $d$. After this, for each of the root-to-leaf paths $\Pi_{1:d}$ with the ground-truth function set $F_{\Pi_{1:d}}$, there are at least ${\rm PMS}_{(d,\gamma)}(\tilde{\H})$ steps where $\delta_y(F_{\Pi_{1:d}}( \T(\Pi_{1:i-1}) )) \le 1 - \gamma$ or $<1$ when $\gamma=0$. Since $\delta_y$ is a delta measure, and $\gamma \in [0,1)$, we have that $y \notin F_{\Pi_{1:d}}( \T(\Pi_{1:i-1}))$ in these steps. Each of such paths $\Pi_{1:d}$ corresponds to a path $\y_{1:d}$ where edges are indexed by labels. by replacing each $\delta_y$ with $y$. Thus, we can build a new tree $\T^{\prime}$ that is ${\rm PMS}_{(d,\gamma)}(\tilde{\H})$-shattered by $\tilde{\H}$ from the pruned $\T$. Concretely, we remain the nodes and annotation functions unchanged, and change each edge from $\delta_y$ to $y$. For each of the root-to-leaf paths, we arrange the same ground-truth function set. It is not hard to see that $\T^{\prime}$ is ${\rm PMS}_{(d,\gamma)}(\tilde{\H})$-shattered by $\tilde{\H}$. By the definition of PFLdim, it implies ${\rm PFL}_{d}(\tilde{\H})\ge{\rm PMS}_{(d,\gamma)}(\tilde{\H})$.

Finally, we show the existence.
\setcounter{theorem}{27}

\begin{example}
\label{exp:sufficiennotnecessary}
    Let $\X=\N_+$, $\Y=\N_+$, $\S(\Y) = \{ S_i = \{ i \}^{\infty}_{i=n} \ \mid\ n \in \N_+ \}$, and $\H=\{f_{i_1,i_2,\dots} (x) = i_x \ \mid\ i_1,i_2,\dots \in \N_{+}  \}$. To see $\mathscr{H}(\S(\Y)) = \infty$, simply note that for any of the finite subcollection $\S^{\prime}$ of $\S(\Y)$, we have that $\bigcap_{S\in \S^{\prime}} S \neq \emptyset$, and $\bigcap^{\infty}_{i=1} S_i = \emptyset$. Meanwhile, $\S(\Y)$ satisfies that $S_1 \supseteq S_2 \supseteq \dots$. Thus, this is an example where the deterministic and randomized algorithms are inseparable, while $\mathscr{H}(\S(\Y)) = \infty$. In this concrete setting, the deterministic and randomized algorithms are both unlearnable.
\end{example}
\vspace{-0.7cm}
\end{proof}

\section{Proof of Separation between Oblivious and Public Settings}

\setcounter{theorem}{17}
\begin{theorem}[Oblivious=Public for Set-valued Setting]
\label{thm:o=p}
    Fix $\X$, $\Y$, $\S(\Y)$, and $\H$. The existence-realizable online learnability with set-valued feedback and randomized algorithms of $\H$ is identical under oblivious and public settings.
\end{theorem}
\begin{proof}
We first recall and rewrite the minimax regret of the public setting in a sequential game style,
\begin{equation}
\label{eq:minimaxregretofpublic}
\begin{aligned}
    R^{\rm pub}_{\rm sv}(T,\H) &= \sup_{x_1}\inf_{\hat\pi_1}\sup_{S_1} \underset{\hat y_{1}\sim \hat \pi_1}{\E} 
\sup_{x_2}\inf_{\hat\pi_2}\sup_{S_2} \underset{\hat y_{2}\sim \hat \pi_2}{\E}  \cdots \sup_{x_T}\inf_{\hat\pi_T}\sup_{S_T} \underset{\hat y_{T}\sim \hat \pi_T}{\E}  \\ &\qquad \sum_{t=1}^T \1\{\hat y_t\notin S_t\}, \qquad \text{where} \ \ \exists\,f^\star\in \H\ \text{s.t.}\ f^\star(x_i)\in S_i, S_i\in \S(\Y),\ \forall i\in[T],
\end{aligned}
\end{equation}
and the one of the oblivious setting,
\begin{equation}
\label{eq:minimaxregretofoblivious}
\begin{aligned}
    R^{\rm obl}_{\rm sv}(T,\H) &= \sup_{x_1}\inf_{\hat\pi_1}\sup_{S_1} 
\sup_{x_2}\inf_{\hat\pi_2}\sup_{S_2}   \cdots \sup_{x_T}\inf_{\hat\pi_T}\sup_{S_T} \underset{\hat y_{i}\sim \hat \pi_i,\forall i\in [T]}{\E}  \\ &\qquad  \sum_{t=1}^T \1\{\hat y_t\notin S_t\},\qquad \text{where} \ \ \exists\,f^\star\in \H\ \text{s.t.}\ f^\star(x_i)\in S_i, S_i\in \S(\Y),\ \forall i\in[T].
\end{aligned}
\end{equation}
The goal is to prove that Equations~\eqref{eq:minimaxregretofpublic} and \eqref{eq:minimaxregretofoblivious} are equal.
To prove the equality, it suffices to show that in the set-valued feedback setting,
the continuation minimax value after round $t$ depends only on the revealed pair-history
$(x_1,S_1,\dots,x_t,S_t)$, and not on the realized sampled labels
$\hat y_1,\dots,\hat y_t$.

For any realizable prefix
\[
\h_t := (x_1,S_1,\dots,x_t,S_t),
\]
define the oblivious continuation value recursively by
\[
V_T(\h_T):=0,
\]
and for $t=T-1,\dots,0$,
\[
V_t(\h_t)
:=
\sup_{x_{t+1}}
\inf_{\hat\pi_{t+1}}
\sup_{\substack{S_{t+1}\in \S(\Y):\\
\h_t\sqcup(x_{t+1},S_{t+1})\text{ is realizable}}}
\Bigl(
\hat\pi_{t+1}(\Y\setminus S_{t+1})
+
V_{t+1}(\h_t\sqcup(x_{t+1},S_{t+1}))
\Bigr).
\]
By definition,
\[
V_0 = R^{\rm obl}_{\rm sv}(T,\H).
\]
Now, for the public setting, let
\[
\tau_t := (x_1,S_1,\hat y_1,\dots,x_t,S_t,\hat y_t)
\]
be a realized public history, and let $W_t(\tau_t)$ be the minimax regret value
of the total loss given the realized public history $\tau_t$. We claim that for every $t\in\{0,1,\dots,T\}$,
\begin{equation}
\label{eq:public-oblivious-claim}
W_t(\tau_t)
=
\sum_{i=1}^t \mathbf{1}\{\hat y_i\notin S_i\}
+
V_t(\h_t).
\end{equation}

We prove \eqref{eq:public-oblivious-claim} by backward induction on $t$.

For $t=T$, the game has ended, so
\[
W_T(\tau_T)=\sum_{i=1}^T \mathbf{1}\{\hat y_i\notin S_i\}
=
\sum_{i=1}^T \mathbf{1}\{\hat y_i\notin S_i\}+V_T(\h_T),
\]
since $V_T(\h_T)=0$. Thus the claim holds.

Assume now that the claim holds for $t+1$. Then
\[
\begin{aligned}
W_t(\tau_t)
&=
\sup_{x_{t+1}}
\inf_{\hat\pi_{t+1}}
\sup_{\substack{S_{t+1}\in \S(\Y):\\
\h_t\sqcup(x_{t+1},S_{t+1})\text{ is realizable}}}
\mathbb{E}_{\hat y_{t+1}\sim \hat\pi_{t+1}}
\Bigl[
W_{t+1}(\tau_t,x_{t+1},S_{t+1},\hat y_{t+1})
\Bigr].
\end{aligned}
\]
Applying the induction hypothesis at round $t+1$ gives
\[
\begin{aligned}
W_t(\tau_t)
&=
\sum_{i=1}^t \mathbf{1}\{\hat y_i\notin S_i\}
+
\sup_{x_{t+1}}
\inf_{\hat\pi_{t+1}}
\sup_{\substack{S_{t+1}\in \S(\Y):\\
\h_t\sqcup(x_{t+1},S_{t+1})\text{ is realizable}}}
\mathbb{E}_{\hat y_{t+1}\sim \hat\pi_{t+1}}
\Bigl[
\mathbf{1}\{\hat y_{t+1}\notin S_{t+1}\}
\\
&\hspace{5cm}
+
V_{t+1}(\h_t\sqcup(x_{t+1},S_{t+1}))
\Bigr].
\end{aligned}
\]
Since $V_{t+1}(\h_t\sqcup(x_{t+1},S_{t+1}))$ does not depend on the realized
sample $\hat y_{t+1}$, the expectation only acts on the indicator term. Hence
\[
\begin{aligned}
W_t(\tau_t)
&=
\sum_{i=1}^t \mathbf{1}\{\hat y_i\notin S_i\}
+
\sup_{x_{t+1}}
\inf_{\hat\pi_{t+1}}
\sup_{\substack{S_{t+1}\in \S(\Y):\\
\h_t\sqcup(x_{t+1},S_{t+1})\text{ is realizable}}}
\Bigl(
\hat\pi_{t+1}(\Y\setminus S_{t+1})
+
V_{t+1}(\h_t\sqcup(x_{t+1},S_{t+1}))
\Bigr)
\\
&=
\sum_{i=1}^t \mathbf{1}\{\hat y_i\notin S_i\}
+
V_t(\h_t).
\end{aligned}
\]
Thus \eqref{eq:public-oblivious-claim} holds for all $t$.

Finally, taking $t=0$, we obtain
\[
R^{\rm pub}_{\rm sv}(T,\H)=W_0=V_0=R^{\rm obl}_{\rm sv}(T,\H).
\]
Therefore the minimax regrets in the public and oblivious settings are equal, and in
particular the existence-realizable online learnability with set-valued feedback is identical
under the two settings.
\end{proof}

\begin{theorem}[Oblivious$\neq$Public for Partial-Feedback Learning]
\label{thm:oneqpforpartial}
    There exists a choice of $\X$, $\Y$, $\S(\Y)$, and $\H$ such that $\H$ is partial-feedback online learnable under randomized algorithms, set-realizable, and oblivious settings, but unlearnable in the public setting. 
\end{theorem}
\begin{proof}
    We present such an example.
\setcounter{theorem}{27}
\begin{example}
\label{exp:cube}
    Let $\X=\Y=\N_+$, $\H=\{f_{i_1,i_2,\dots} (x) = i_x \ \mid\ i_1,i_2,\dots \in \N_{+}  \}$, and $\S(\Y) = \{S_y=\Y\setminus\{y\} \ \mid\ y\in\Y   \}$. We first show that there exists a randomized algorithm that can learn $\H$ under oblivious setting. Suppose the game will go for $T$ rounds. The learner predicts $i$ with possibility $\frac{1}{T}$ for each $i\in [T]$. No matter what strategy the adversary adopts, one set in $S(\Y)$ must be taken as the ground-truth set in each round. Under oblivious setting, in each round, the actual predicted label is sampled from the output measure after the ground-truth set is chosen. The learner will fail a round only when the prediction equals the only label excluded from the ground-truth set. Thus, the expectation of the error for each step is at most $\frac{1}{T}$, and the overall regret is no more than $1$. Next, we show that any of the algorithm under public setting cannot learn $\H$. It suffices to show a concrete adversary that there exists a constant $1>k>0$, such that in every round, the adversary can force any algorithm to incur loss higher than $k$. In round $t$, let the adversary choose $x=t$, and no matter what measure the learner predicts, say $\hat{\pi}_t$, the adversary reveals some $y^{\rm vis}_t$, such that $\pi_t(\{y^{\rm vis}_t\})\le 1-k$ with any fixed $k$ between $0$ and $1$, which is possible for $\Y = \N_+$. In the public setting, the adversary determines the ground-truth function sets in each round after the game is done and all predicted labels are fixed. For those rounds that the learner's prediction mismatches the revealed label, the adversary assigns the correct label set in each round to be all the labels in $\Y$ but the one that the learner predicts. For those matched rounds, the adversary assigns the correct label set in each round to be all the labels in $\Y$ but any one of labels not equaling the learner's prediction. The associated ground-truth function set is set to be $F:=\H\setminus (\bigcup_{t=1}^{T}\mathcal{M}_t)$, where $\M_t :=  \{f_{\dots,i_{t-1}, \hat{y}_t,i_{t+1},\dots}\ \mid \ i_1,i_2,\dots \in \N_+ \}$ for those unmatched rounds, and $\M_t :=  \{f_{\dots,i_{t-1}, y_{\setminus \hat{y}_t},i_{t+1},\dots}\ \mid \ i_1,i_2,\dots \in \N_+ \}$ where $y_{\setminus \hat{y}_t}$ denotes some number other than $\hat{y}_t$ for those matched rounds. It is not hard to verify that this guarantees a linear rate of regret, $\Theta(T)$.
\end{example}
\vspace{-0.7cm}
\end{proof}

\begin{table}[]
\centering
\resizebox{\linewidth}{!}{\begin{tabular}{rcccccl}
\toprule
\multicolumn{7}{c}{Partial-Feedback}                                                                                                                                                                                                                                                                                            \\ \midrule
Action & $\Q$ sel. $x_t$     & $\A$ pred. $\hat{\pi}_t$              & $\Q$ rev. $y^{\vis}_t$                        & $\hat{y}_t$ drawn on $\hat{\pi}_t$                & \makecell{(if $t=T$) \\ $\Q$ rev. $(S_1,\dots,S_T)$} & \multicolumn{1}{c}{\makecell{(if $t=T$) \\ $\ell_T := \sum^T_{i=1} \1\{ \hat{y}_i\notin S_i \}$}} \\ \midrule
Info   & $\Theta_{t-1}$          & $\Theta_{t-1} \bigcup \{ x_t \}$          & $\Theta_{t-1} \bigcup (x_t,\hat{\pi}_t)$          & $\Theta_{t-1} \bigcup (x_t,\hat{\pi}_t,y^{\vis}_t)$   & $\Theta_T$                                            & \multicolumn{1}{c}{$\Theta_T \bigcup (S_1,\dots,S_T)$}                                      \\ \specialrule{1pt}{3pt}{3pt}
\multicolumn{7}{c}{Multiclass~\citep{littlestone} \& Set-valued~\citep{onlinelearningsetvaluefeedback}}                                                                                                                                                                                                                                                                    \\ \midrule
Action & $\Q$ sel. $x_t$     & $\A$ pred. $\hat{\pi}_t$              & $\Q$ rev. $S_t$                               & $\hat{y}_t$ drawn on $\hat{\pi}_t$                & $\ell_t:=\ell_{t-1} + \1\{ \hat{y}_t \notin S_t \} $       &                                                                                              \\ \midrule
Info   & $\Theta^{\prime}_{t-1}$ & $\Theta^{\prime}_{t-1} \bigcup \{ x_t \}$ & $\Theta^{\prime}_{t-1} \bigcup (x_t,\hat{\pi}_t)$ & $\Theta^{\prime}_{t-1} \bigcup (x_t,\hat{\pi}_t,S_t)$ & $\Theta^{\prime}_{t}$                               &                                                                                              \\ \bottomrule
\end{tabular}}
\caption{The learning protocol of partial-feedback compared with multiclass \& set-valued protocols under public setting, where $\A$ is learner, $\Q$ is adversary, histories $\Theta_{i}:=(x_1,\hat{\pi}_1,y^{\vis}_1,\hat{y}_1),\dots,(x_{i},\hat{\pi}_{i},y^{\vis}_{i},\hat{y}_{i})$, and $\Theta^{\prime}_{i}:=(x_1,\hat{\pi}_1,S_1,\hat{y}_1),\dots,(x_{i},\hat{\pi}_{i},S_{i},\hat{y}_{i})$.}
\label{tab:settingpublic}
\end{table}

\end{document}